\documentclass[11pt]{article}

\usepackage[preprint]{acl}

\usepackage{times}
\usepackage{latexsym}

\usepackage[T1]{fontenc}

\usepackage[utf8]{inputenc}

\usepackage{microtype}

\usepackage{inconsolata}

\usepackage{graphicx}

\usepackage[utf8]{inputenc} 
\usepackage[T1]{fontenc}    
\usepackage{hyperref}       
\usepackage{url}            
\usepackage{booktabs}       
\usepackage{amsfonts}       
\usepackage{nicefrac}       
\usepackage{microtype}      
\usepackage{xcolor}         

\usepackage{makecell}
\usepackage{lineno}
\usepackage{svg}
\usepackage{enumitem}
\usepackage[most]{tcolorbox}
\usepackage{xcolor}
\usepackage{enumitem}
\usepackage{pifont}
\usepackage{amssymb}
\usepackage{listings}
\usepackage[normalem]{ulem}
\usepackage{placeins}

\definecolor{CardBlue}{HTML}{08045C}
\definecolor{SoftCard}{HTML}{F7F8FF}
\definecolor{CorrectGreen}{HTML}{167A3A}
\definecolor{WrongRed}{HTML}{B42318}

\tcbset{
  datasetcardstyle/.style={
    enhanced,
    colback=SoftCard,
    colframe=CardBlue,
    colbacktitle=CardBlue,
    coltitle=white,
    boxrule=0.4pt,
    arc=1.8mm,
    left=1.2mm,
    right=1.2mm,
    top=0.9mm,
    bottom=0.9mm,
    fonttitle=\bfseries\scriptsize,
    fontupper=\scriptsize,
    before skip=0pt,
    after skip=0pt,
  }
}

\newcommand{\gt}[1]{\textcolor{CorrectGreen}{\textbf{Ground Truth: #1}}}
\newcommand{\modelwrong}[1]{\textcolor{WrongRed}{\textbf{Model: #1 \ding{55}}}}

\newcommand{\haoyue}[1]{\textcolor{orange}{HB: #1}}

\usepackage{amsmath,amsfonts,bm}

\def\eqref#1{equation~\ref{#1}}

\def\1{\bm{1}}

\DeclareMathAlphabet{\mathsfit}{\encodingdefault}{\sfdefault}{m}{sl}
\SetMathAlphabet{\mathsfit}{bold}{\encodingdefault}{\sfdefault}{bx}{n}

\newtheorem{theorem}{Theorem}[section]

\newtheorem{assumption}[theorem]{Assumption}

\newtheorem{proof}[theorem]{Proof}

\title{Understanding Why Language Models Hallucinate:\\Testing Reasoning Against Priors}

\author{
\textbf{Yangfan Hu}\thanks{Equal contribution. Author order was randomly determined.}\quad
\textbf{Xuhan Tong}$^{*}$\quad
\textbf{Haoyue Bai}$^{*}$\thanks{Corresponding authors.}\\
\textbf{Xi Ding}\quad \textbf{Shashank Muralidhar Bharadwaj} \quad \textbf{Siyang Cao}\\
\textbf{Robert Nowak}\quad
\textbf{Jiawei Zhang}\footnotemark[2]\\[3pt]
Department of Computer Sciences, University of Wisconsin--Madison\\
1205 University Avenue, Madison, WI 53706, USA\\
\{haoyue.bai,jiawei.zhang\}@cs.wisc.edu
}

\begin{document}

\maketitle

\begin{abstract}
Large language models often produce hallucinated answers that violate prompt-level
constraints. A key diagnostic question is
whether these failures reflect missing knowledge, or whether the model has the
relevant information but follows the wrong inference path.
We study this phenomenon as \emph{inference misalignment}: a mismatch between
the answer supported by the prompt and the answer favored by statistically
salient latent associations. We formalize this view with a latent key--task
model, in which pretraining-frequency imbalance can cause a shortcut path to
dominate the constraint-sensitive path and induce positive inference loss. The
framework predicts two failure modes: task-retrieval bias in entity
disambiguation and key-selection bias in action choice.
We introduce \textsc{TrapQA}, a controlled diagnostic testbed with two
components. \textsc{Scientist QA} tests disambiguation among similar scientists
with supplementary factual probes, while \textsc{Real-Life Constrained QA}
tests everyday constraint following under salient shortcuts. Our results provide behavioral evidence that some hallucinations
reflect inference misalignment rather than absent knowledge alone. \footnote{The project page is available \href{https://neohughus.github.io/Understanding_Why_Language_Models_Hallucinate/}{here.}}
\end{abstract}

\section{Introduction}

Large language models (LLMs) have achieved strong performance across many tasks
\citep{openai2024gpt4technicalreport,geminiteam2025geminifamilyhighlycapable,
deepseekai2025deepseekv3technicalreport,grattafiori2024llama3herdmodels}.
They are also increasingly integrated with tools and agentic workflows, such as
web search and external services
\citep{nakano2022webgptbrowserassistedquestionansweringhuman,
liu2023webglmefficientwebenhancedquestion,steinberger2026introducingopenclaw}.
As model outputs become more tightly coupled to real-world actions,
hallucination remains a central reliability risk.

Hallucination broadly refers to fluent but factually incorrect, unsupported, or
context-unfaithful outputs \citep{Ji_2023,Huang_2025}. Such errors are difficult
to detect when models sound confident or when users lack domain expertise, and
they can be amplified in agentic settings through downstream tool calls or
transactions. Importantly, hallucinations can arise even under benign inputs,
making them an intrinsic reliability problem rather than only a failure under
adversarial attack \citep{zhang2025sirenssongaiocean,Huang_2025}.

Prior work studies hallucination through training-data bias, decoding dynamics,
and attribution or mechanistic analysis
\citep{dziri-etal-2022-origin,zhang2023languagemodelhallucinationssnowball,
sun2025llm,gao2025hneuronsexistenceimpactorigin}.
Existing evaluations such as TruthfulQA and HaluEval measure important aspects
of truthfulness and hallucination behavior
\citep{lin-etal-2022-truthfulqa,li-etal-2023-halueval}. However, a central
mechanistic question remains underexplored: when a model fails, did it lack the
needed knowledge, or did it possess the relevant facts but retrieve and apply the
wrong inference path?

We address this question by interpreting hallucination as \emph{inference
misalignment}: a mismatch between the answer logically supported by the prompt
and the answer favored by statistically salient learned associations. In our
framework, a prompt activates latent key--task paths. A model hallucinates when a
high-frequency shortcut path receives greater posterior weight than the
constraint-sensitive path required by the prompt. This view predicts that errors
can occur even when the relevant facts or constraints are available: the failure
lies not only in stored knowledge, but in selecting and composing the appropriate
inference path. Figure~\ref{fig:scientist_failure_illustration} illustrates this competition and the
paired-probe diagnostic.

Guided by this theory, we introduce \textsc{TrapQA}, a closed-book diagnostic
benchmark suite with two complementary settings. \textsc{Scientist QA} targets
\emph{task-retrieval bias}, where a salient entity--relation association
overrides a discriminative constraint. \textsc{Real-Life Constrained QA} targets
\emph{key-selection bias}, where a superficially salient cue dominates the
task-relevant constraint. These settings are designed not as broad-coverage
benchmarks, but as controlled tests of the latent key--task framework. External
tools, including web search, are disabled in all evaluations.

\textsc{Scientist QA} tests constraint-sensitive disambiguation among highly
similar scientists. For each pair, we generate a biography-style paragraph
broadly compatible with both candidates, then append a decisive fact that rules
out exactly one candidate. The model must choose the scientist matching the full
description. We also include closed-book probes for candidate-specific facts,
allowing us to distinguish missing knowledge from knowledge-deployment failure.
Across 2,925 questions and eight model--reasoning configurations spanning GPT,
Gemini, Claude, and DeepSeek, hallucination rates range from 2.50\% to 37.23\%
in the retrieval-sensitive names-only setting. Many errors persist even when the
model answers the relevant probe facts correctly in isolation, indicating a
failure of comparative knowledge deployment rather than factual ignorance alone.
\begin{figure*}[t]
    \centering
    \includegraphics[width=\linewidth]{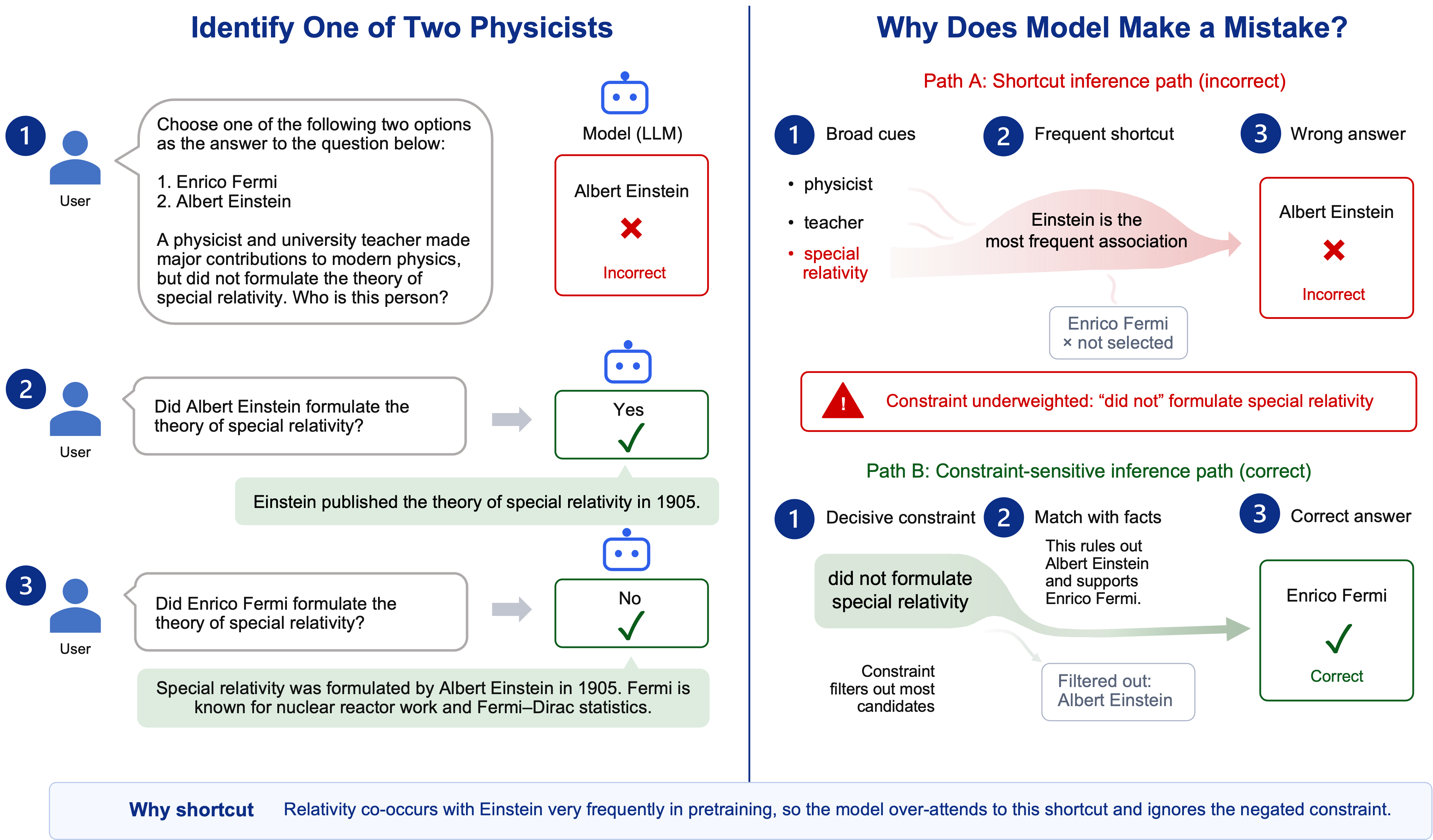}
    \caption{
    Hallucination as misaligned inference. In the motivating \textsc{Scientist QA}
    example, the model selects the wrong scientist by following a high-salience
    association while underweighting the decisive discriminative constraint.
    The probes show that the relevant facts are behaviorally accessible in
isolation; the right panel is a conceptual interpretation of the resulting
answer--probe inconsistency, not an observed internal reasoning trace.
    }
    \label{fig:scientist_failure_illustration}
\end{figure*}

\textsc{Real-Life Constrained QA} complements \textsc{Scientist QA} by testing
shortcut failures in everyday action choice. We begin from lexical associations
in \textsc{SWOW} \citep{de2019small}, use them as high-salience shortcut cues,
and organize them through eight seed template families such as
\texttt{vehicle\_required}, \texttt{delivery\_medium},
\texttt{recording\_medium}, and \texttt{tool\_required}. GPT instantiates
natural forced-choice scenarios in which one option is superficially tempting
but violates a prompt-grounded physical, spatial, procedural, or medium-specific
constraint. After filtering and controlled perturbations, the final collection
contains 500 questions covering 13 aspects of daily life. Claude, GPT, Gemini,
and DeepSeek make 81, 44, 18, and 182 mistakes, respectively, corresponding to
error rates of 16.2\%, 8.8\%, 3.60\%, and 36.4\%. These failures show that
inference misalignment is not limited to encyclopedic entity disambiguation: it
also appears when models must select an action under ordinary real-world
constraints.

\paragraph{Contributions.}
We make three contributions. \textbf{(1)} We propose a latent key--task
framework that formalizes hallucination as inference misalignment caused by
posterior dominance of statistically salient shortcut paths over
prompt-supported constraint-sensitive paths. \textbf{(2)} We derive theoretical
predictions linking pretraining-frequency imbalance to shortcut posterior
and positive inference loss, providing a mechanistic account of why
hallucinations can occur even under benign prompts. \textbf{(3)} We introduce
\textsc{TrapQA}, a closed-book diagnostic benchmark suite with two complementary
settings, and evaluate frontier model families with external tools disabled,
showing that hallucinations can persist despite isolated factual knowledge and
under prompt-grounded real-world constraints.

\section{Related Work}
\label{sec:related_work}

Hallucination in language generation predates the recent LLM era. Neural data-to-text systems can produce fluent outputs that fail to reflect the underlying records \citep{wiseman-etal-2017-challenges}; neural machine translation models may generate plausible but source-unsupported translations \citep{lee2019hallucinations,raunak2021}; and abstractive summarization models often produce content that is not faithful to the input document \citep{maynez-etal-2020-faithfulness}. Across these settings, hallucination reflects a common failure mode: the model produces plausible language that is insufficiently grounded in the input, retrieved evidence, or relevant knowledge.

Recent work studies hallucination from both theoretical and empirical perspectives. Formal accounts show that hallucination can arise from fundamental learning or statistical limitations \citep{xu2025,kalai2024}, while mechanistic accounts trace hallucinations to competing associations or latent-state dynamics during generation \citep{sun2025llm,cherukuri2026}. LUMINA~\citep{yeh2026lumina} analyzes hallucinations in RAG systems by measuring the utilization of model's internal knowledge and retrieved external context. HalluGuard~\citep{zeng2026halluguard} decomposes hallucination risk into data-driven and reasoning-driven components and proposes an NTK-based score for jointly detecting both types of hallucinations. Other work locates hallucination across the LLM pipeline, including distributional imbalance and noise in pretraining data, difficulty acquiring new factual knowledge during fine-tuning, and inference-time reliance on memorized or frequency-biased patterns \citep{zhang2025measuring,liu2026pretrainrl,gekhman2024,zhang2023languagemodelhallucinationssnowball,mckenna2023,berglund2024}. These suggest that hallucination is not merely a matter of missing knowledge; it can also reflect failures in retrieving, comparing, or applying knowledge under the constraints of a particular prompt. Appendix~\ref{app:related_work_extended} provides a fuller discussion, including reinforcement-learning-based mitigation.

A large body of benchmarks evaluates factuality and faithfulness in generated text, including TruthfulQA and HaluEval \citep{lin-etal-2022-truthfulqa,li-etal-2023-halueval}, long-form factuality benchmarks such as FActScore and LongFact \citep{min2023factscore,wei2024long}, short-form factuality benchmarks such as SimpleQA \citep{wei2024measuring}, retrieval-augmented generation benchmarks such as RAGTruth and FRAMES \citep{niu2024ragtruth,krishna2025}. 
Recent work has also developed methods and benchmarks specifically for hallucination detection, including HaloScope, which leverages unlabeled LLM generations, Reasoning Denoiser, which denoises reasoning traces for hallucination detection in large reasoning models, and OpenHalDet, which evaluates hallucination detection across diverse generation scenarios \citep{du2024haloscope,fang2026reasoning,li2026openhaldet}. These approaches primarily measure whether an output is factual, faithful, or grounded in evidence. Recent detection approaches further operate at different levels of representation: HalluEntity evaluates entity-level uncertainty signals \citep{yeh2025halluentity}, \citet{park2025steer} steers latent representations to improve separation between truthful and hallucinated generations, and GENUINE incorporates dependency structure into uncertainty estimation \citep{wang-etal-2025-genuine}.
Relatedly, ProbeLLM \citep{huang2026probellm} automatically discovers structured LLM failure modes through benchmark-agnostic probing. 
Our goal is diagnostic, but targets a different question: Scientist QA and Real-Life Constrained QA are designed to isolate why a model fails. Scientist QA tests whether models can deploy candidate-specific facts under disambiguation, while Real-Life Constrained QA tests whether models can override SWOW-derived associative cues with prompt-grounded physical, spatial, procedural, or medium-specific constraints \citep{de2019small}. This design allows us to distinguish simple ignorance from knowledge-deployment failures in which relevant information is available to the model but not used in the relevant inference path.

\section{Pretraining Frequency Induces Hallucination in Latent Inference}
\label{sec:freq}

\subsection{Hallucination as Inference Misalignment}

In the study of LLM reliability, \emph{hallucination} is commonly described as the generation of factually incorrect or logically inconsistent content. However, to enable a quantitative mathematical analysis, we argue that the underlying mechanism of hallucination can be more precisely characterized as \emph{inference misalignment}.

Specifically, let $\bm z$ denote a prompt sequence of bounded length. A pretrained
sequence model induces a conditional distribution $P(\cdot\mid \bm z)$ over
continuations. We further assume the existence of an ideal predictor defining the
ground-truth conditional distribution $P_\star(\cdot\mid \bm z)$, representing
the correct reasoning process for the task associated with the prompt. For analysis, we work in the embedding space where smoothness properties are well defined. Conceptually, the model’s inference can be viewed as moving from regions well-supported by the pretraining distribution
toward the test prompt along an optimal \emph{inference path}, producing an
output consistent with $P_\star(\cdot\mid \bm z)$ when the semantic structure
and task intent are correctly identified.

However, due to statistical regularities present in the pretraining corpus, the model may instead rely on high-frequency \emph{shortcuts}. These shortcuts bias the model toward inference paths that are statistically dominant but semantically incorrect. Consequently, the model may traverse regions of the representation space that are weakly supported by the training data, leading to unstable behavior.

Guided by this perspective, we define the inference loss as the discrepancy between the model's predictive distribution and the ideal distribution: $
\ell(\bm{z}) := \|P(\cdot|\bm z)- P_\star(\cdot|\bm z)\|_{TV}$,
where $\|\cdot\|_{TV}$ denotes the total variation distance.
Through this formalization, we shift the study of hallucination from heuristic descriptions of textual inconsistency to a principled analysis of inference behavior.

\subsection{Latent Key--Task Model of LLM Inference}
\label{subsec:kt}
To analyze how a model generalizes from the pretraining corpus to a new prompt, we introduce a conceptual abstraction of the inference stage. Let $\mathcal{T} = \{t_1,\dots,t_p\}$ denote a finite set of latent tasks and let
$\mathcal{K}=\{k_{1},\dots,k_{m}\}$ denote a finite set of task-informative
\emph{keys}. A key represents a salient feature pattern in a prompt (for example,
a lexical cue or structural pattern) that provides evidence about the underlying task. We model LLM inference as an implicit two-stage reasoning process
\begin{align*}
\label{eq:flow}
\bm z \;\xrightarrow{\text{identify key}}\; k_{i}
\;\xrightarrow{\text{retrieve task}}\;
t_{j}(k_{i})
\;\xrightarrow{\text{generate}}\; y .
\end{align*}

Following the Bayesian perspective of \citet{Xie2022}, we model the
model's prompt processing as an implicit inference over latent
variables. Under this view, the model associates each prompt with a posterior distribution over key--task pairs and combines their contributions to form the prediction. Given a prompt $\bm z$, the model implicitly forms a posterior distribution
$P(k,t|\bm z)$ over the latent space and aggregates predictions along these hypotheses:
\[
P(y|\bm z)
=
\sum_{k\in\mathcal K,\,t\in\mathcal T}
P(k,t|\bm z)\,P(y|\bm z;k,t),
\]
where $P(y|\bm z;k,t)$ denotes the predictive distribution under the
hypothesis that the prompt corresponds to key $k$ and task $t$.

\subsection{Frequency-Induced Bias in Inference}
We now introduce a framework that explicates how the statistical imbalance of the pretraining corpus leads to inference shortcuts and formally induces hallucination through generalization error.

\paragraph{Pretraining Statistics.} 
The distribution of keys and tasks in the pretraining corpus induces a statistical prior that guides latent inference. Let $c_i^{(k)}$ denote the number of occurrences of key $k_i$ in the corpus, and let
$C^{(k)}=\sum_{i=1}^{m} c_i^{(k)}$. Conditioned on a key $k$, let $c_j^{(t)}(k)$ denote the number of occurrences in which task $t_j$ co-occurs with $k$, and let
$C^{(t)}(k)=\sum_{j=1}^{p} c_j^{(t)}(k)$. The empirical key distribution and conditional task distribution are
\[
\pi^{(k)}(k_i)=\frac{c_i^{(k)}}{C^{(k)}}, \quad \pi^{(t)}(t_j|k)=\frac{c_j^{(t)}(k)}{C^{(t)}(k)}.
\]
Throughout the analysis, we use these empirical frequencies as proxies for the underlying pretraining probabilities.
These statistics define a joint prior over key--task pairs
$\pi(k,t)=\pi^{(k)}(k)\,\pi^{(t)}(t|k)$.

\paragraph{Event-based Perspective.} 
We now specialize the latent inference framework, where the model is asked to select between two candidate entities or actions. In this setting, the prompt $\bm z$ explicitly presents two candidate keys $k^\ast$ and $k_s$: $k^\ast$ is the candidate consistent with the prompt's decisive constraint and corresponds to the correct answer, while $k_s$ is the alternative candidate. We refer to $k_s$ as the \emph{shortcut key} when its associated key--task pair has higher pretraining frequency than that of $k^\ast$, so that statistical salience favors $k_s$ even though the prompt-level evidence favors $k^\ast$. Correspondingly, let 
$\{t^\ast, t_s\}$ denote the correct and the shortcut tasks.

\begin{assumption}[Activated Key Restriction]\label{ass:event}
For a prompt $\bm z$ presenting candidates $k^\ast$ and $k_s$:
\begin{enumerate}[leftmargin=*]
\item Negligible mass outside the candidate pair. Latent keys other than $k^\ast$ and $k_s$ receive vanishing posterior: $
P\!\left(k \notin \{k^\ast, k_s\} \,\middle|\, \bm z\right) \;\ll\; 1$.
\item Prior-driven posterior on the candidate pair. Within the candidate pair, the prompt does not differentially update the relative posterior of the two keys:
$
\bm z \perp k\mid \{k \in \{k^\ast, k_s\}\}.
$
\end{enumerate}
\end{assumption}

A useful way to interpret Assumption~\ref{ass:event} is that (i) candidate keys outside the explicitly presented pair receive negligible posterior support; (ii) once the candidate set is fixed by the prompt template, the remaining prompt content does not, on its own, alter the relative plausibility of the two activated keys at the level of latent identification. The model's choice between $k^\ast$ and $k_s$ is then driven by the pretraining prior over keys rather than by within-prompt likelihood asymmetries.

\begin{assumption}[Activated Task Restriction]\label{ass:event_task}
For each path, let $t^\ast$ and $t_s$ denote the tasks associated with $k^\ast$ and $k_s$ respectively.
\begin{enumerate}[leftmargin=*]
\item Negligible mass outside the candidate tasks. Conditional on the activated key being $k^\ast$ (resp.\ $k_s$), the task posterior concentrates on the candidate set $\{t^\ast, t_s\}$:
$P\!\left(t \notin \{t^\ast, t_s\} \,\middle|\, \bm z, k^\ast\right) \ll 1$ and
$ P\!\left(t \notin \{t^\ast, t_s\} \,\middle|\, \bm z, k_s\right) \ll 1$.
\item Prior-driven posterior on the candidate tasks. Within the candidate task pair, the prompt does not differentially update the relative posterior of the two tasks given the activated key: $
\bm z \perp t\mid k, \{t\in \{t^\ast, t_s\}\}$.
\end{enumerate}
\end{assumption}

\begin{assumption}[Output separation]
\label{ass:pairwise_output_separation}
Let $y^\ast$ denote
the correct answer and let $y_s$ denote the shortcut-induced answer. We assume
that the two paths induce separated predictions:
$P(y^\ast\mid \bm z;k_s,t_s)\ll1$, and $P(y_s\mid \bm z;k^\ast,t^\ast)\ll1$.
Moreover, the shortcut path is at least as confident in the shortcut answer as
the correct path is in the correct answer:
$
P(y_s\mid \bm z;k_s,t_s)
\ge
P(y^\ast\mid \bm z;k^\ast,t^\ast)$.
\end{assumption}

\begin{theorem}[Shortcut Probability Dominance]
\label{thm:posterior}
Under Assumptions~\ref{ass:event}--\ref{ass:pairwise_output_separation}, consider a fixed prompt $\bm z$ and
the two main competing paths $(k^\ast,t^\ast)$ and $(k_s,t_s)$, then
\[
\frac{P(y_s \mid \bm z)}{P(y^\ast \mid \bm z)}
\gtrsim \frac{\pi(k_s, t_s)}{\pi(k^\ast, t^\ast)}
\cdot
\frac{P(y_s \mid \bm z; k_s,t_s)}{P(y^\ast \mid \bm z; k^\ast,t^\ast)}\gtrsim 1.
\]
\end{theorem}

Theorem~\ref{thm:posterior} shows that when the pretraining frequency of the shortcut pair is
sufficiently larger than that of the correct pair, the shortcut posterior can
dominate the correct posterior even if the prompt contains semantically correct
evidence. We can decompose this frequency-induced hallucination into two distinct modes:

\paragraph{I. Key Selection Bias.}
The first term indicates that the model may fail to attend to the correct semantic anchor because a shortcut key appears much more frequently in the pretraining corpus.

Consider the question: \emph{``I want to go to a car wash. The car wash is only 50 meters away. Should I walk there or drive there?''} Many models answer that one should walk, since 50 meters is a very short distance, though without driving the car to the car wash the task cannot be completed. This suggests that the model attends to the statistically dominant distance key (i.e., ``50 meters'') rather than the semantically decisive key (i.e., ``car wash''), since many examples associate short distances with walking and long distances with driving in pretraining. As a result, the shortcut key $k_s$ corresponding to distance-based transportation choice can dominate the correct key $k^\ast$ corresponding to task feasibility (see Appendix~\ref{app:real_life_constrained_qa_details} for the corresponding experiment).

\paragraph{II. Task Retrieval Bias.}
The theorem also shows that even with the correct key, the model may still retrieve the wrong relation if another task strongly dominates that key family in pretraining.

During pretraining, discussions of
special relativity overwhelmingly associate the concept with Albert Einstein. Consider a prompt asking the model to choose between Enrico Fermi and Albert Einstein:
\emph{``A physicist and university teacher made major contributions to modern physics,
but did not formulate the theory of special relativity.''}
Although the explicit constraint ``did not formulate special relativity'' rules out
Einstein and supports Fermi, the model may still answer ``Albert Einstein''. The
reason is that the affirmative shortcut key linking \emph{special relativity} to
Einstein ($k_s$) dominates the rarer negated constraint ($k^\ast$) in the pretraining
distribution. As a result, the model attends to the statistically dominant association
and effectively ignores the negation (see Section~\ref{sec:main_results}).

Within the shortcut-dominance regime analyzed here, hallucination arises only
when a dominant shortcut key--task pair has been learned during pretraining.
If no representative shortcut pattern exists, this posterior-suppression
mechanism does not occur, and the model may instead rely on the information
provided in the prompt. The complete
proof is provided in Appendix~\ref{appsubsec:freq_post_bound}.

\paragraph{When Assumption~\ref{ass:event} does not hold.} Assumption~\ref{ass:event} is best understood as describing the regime in which $k^\ast$ and $k_s$ are \emph{pretraining-independent}: the two candidate keys rarely co-occur in the same context during pretraining, so the model has not internalized any joint structure relating them. In this regime, even though the prompt $\bm z$ presents both keys together, the model cannot leverage $\bm z$ to extract joint information beyond what is already encoded in the marginal priors $\pi^{(k)}(k^\ast)$ and $\pi^{(k)}(k_s)$, the posterior thus degenerates to the prior ratio.

The complementary regime is when $k^\ast$ and $k_s$ have been seen together during pretraining and the model has learned how their relative importance shifts in joint contexts. In this case, the prompt $\bm z$ is not merely a pair of activated keys but a context whose surface form pattern-matches against pretraining co-occurrence statistics that the model has internalized. The model can then use $\bm z$ to identify which of the two keys is task-relevant in this particular context---essentially deploying a learned within-context disambiguation, rather than falling back on marginal frequency. Assumption~\ref{ass:event} is violated, so the posterior departs from the prior ratio, and hallucination may be avoided.

\paragraph{Beyond the two-path specialization.}
The general model in Section~\ref{subsec:kt} permits arbitrary finite sets of keys and
tasks. Assumptions~\ref{ass:event}--\ref{ass:event_task} isolate the simplest nontrivial competition
between one constraint-sensitive path and one shortcut path, and impose
prior-driven posteriors for analytical tractability; they do not assert that
an LLM activates only two hypotheses or that contextual evidence has no
effect more generally. More generally, the activated paths may be partitioned
into constraint-sensitive and shortcut sets. If prompt-dependent evidence
reweights each path's pretraining prior within a bounded factor, then a
sufficiently large aggregate prior-mass advantage for the shortcut set still
yields shortcut posterior dominance. Conversely, stronger contextual evidence
may amplify constraint-sensitive paths, suppress shortcuts, or activate
additional constraint-relevant paths. 
Appendix~\ref{appsubsec:freq_post_bound} provides the corresponding posterior-odds bound.

\subsection{From Shortcut to Inference Loss}
\label{subsec:shortcut_to_loss}
We now
show that when this posterior dominance changes the model's preferred answer, it
directly induces a positive inference loss.
\begin{assumption}[Target preference margin]
\label{ass:target_preference_margin}
The target distribution prefers the correct answer over the shortcut answer:
there exists $\gamma_\star(\bm z)>0$ such that
$\gamma_\star(\bm z)
:=
P_\star(y^\ast\mid \bm z)-P_\star(y_s\mid \bm z)
>0$.
\end{assumption}
\begin{theorem}[Inference-Loss Lower Bound]
\label{thm:pairwise_tv_lower_bound}
Suppose Assumptions~\ref{ass:pairwise_output_separation}
and~\ref{ass:target_preference_margin} hold. If the shortcut posterior dominates
the correct posterior (i.e., $ P(y^\ast\mid \bm z) < P(y_s\mid \bm z)$, suggested by Theorem~\ref{thm:posterior}),
the inference loss measured by total variation satisfies
\[
\ell(\bm z)
\ge
\frac{1}{2}(\gamma(\bm z)+\gamma_\star(\bm z)).\]
where
$\gamma(\bm z):= P(y_s\mid \bm z)-P(y^\ast\mid \bm z)>0$.
\end{theorem}
Theorem~\ref{thm:pairwise_tv_lower_bound} provides a direct connection
between shortcut posterior dominance and inference loss.
Thus, once
frequency-induced posterior dominance reverses the model's preference between
these two answers, the model distribution and the target distribution must differ
by a non-vanishing margin. The complete proof is provided in Appendix~\ref{app:lower_bound}.

\section{Evaluation Setup}
\label{sec:dataset}

We evaluate the latent-inference account in Section~\ref{sec:freq} with two
controlled diagnostic settings. Scientist QA targets task-retrieval bias in
entity disambiguation, while Real-Life Constrained QA targets key-selection bias
in everyday action choice. External tools, including web search, are disabled in
all evaluations. The names-only Scientist QA condition and supplementary probes
are closed-book; the profiles-in-context condition is a retrieval-relaxed
control.

\paragraph{Scientist QA.}
Scientist QA is constructed from Wikipedia-linked scientist profiles \citep{wikipedia_contributors_2026}. We remove
names from structured profiles, embed the remaining attributes with
\texttt{text-embedding-3-small} \citep{openai2024embedding3small}, retain highly
similar scientist pairs under a sparsity-adjusted similarity score, and use
Gemini to generate a shared biographical description plus a decisive constraint
that rules out exactly one candidate. After filtering and removing
invalid items with GPT, the evaluation set contains 2,925 questions. Each item is tested
under a \emph{names-only} prompt, corresponding to \textit{prepend\_names}, and
a \emph{profiles-in-context} prompt, corresponding to
\textit{prepend\_profiles}. We also attach two closed-book probes derived from
the decisive constraint to distinguish missing factual knowledge from failures
to deploy knowledge in pairwise disambiguation. Construction details are in
Appendix~\ref{app:benchmark_details}.

\paragraph{Real-Life Constrained QA.}
Real-Life Constrained QA tests whether models follow prompt-grounded constraints
when a salient associative shortcut suggests the wrong action. We derive
high-salience cues from \textsc{SWOW} \citep{de2019small}, organize generation
around eight seed template families, and use GPT to instantiate natural
two-option scenarios involving physical, spatial, procedural, or medium-specific
constraints. After filtering and controlled perturbations with Claude, the final collection
contains 500 questions covering 13 aspects of daily life. Representative examples are shown in
Figure~\ref{fig:dataset-examples}, with details in
Appendix~\ref{app:real_life_constrained_qa_details}.

\paragraph{Models and scoring.}
For Scientist QA, we evaluate GPT, Claude, Gemini, and DeepSeek under low- and
high-thinking settings where available; for DeepSeek, \texttt{deepseek-chat} and
\texttt{deepseek-reasoner} serve as the non-reasoning and reasoning modes. For
Real-Life Constrained QA, we report GPT, Claude, Gemini, and DeepSeek-chat results;
all runs use default decoding settings unless otherwise
noted. Prompt templates, model versions, parsing rules, and answer
normalization are reported in Appendix~\ref{app:implementation_details}. A
response is correct if, after normalization, it matches the ground-truth
candidate or action; choosing the ruled-out Scientist QA candidate, producing an
off-option Scientist QA response, or selecting the shortcut Real-Life action is
counted as an error.

\begin{figure}[t]
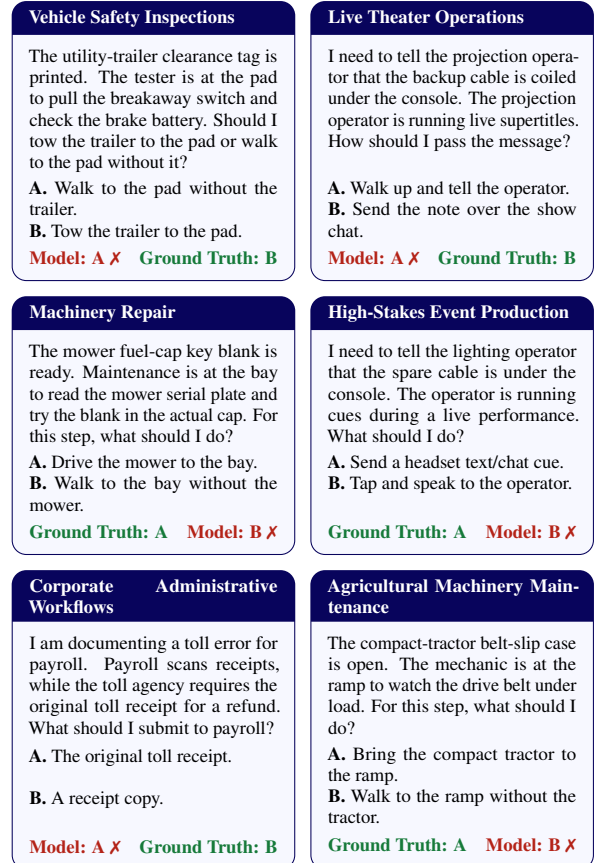

\centering

\begin{tcbitemize}[
  datasetcardstyle,
  raster columns=2,
  raster equal height=rows,
  raster column skip=2mm,
  raster row skip=2mm,
  raster width=\linewidth
]

\tcbitem[title={Vehicle Safety Inspections}]
The utility-trailer clearance tag is printed. The tester is at the pad to pull the breakaway switch and check the brake battery. Should I tow the trailer to the pad or walk to the pad without it?

\vspace{0.25em}
\textbf{A.} Walk to the pad without the trailer.\\
\textbf{B.} Tow the trailer to the pad.

\vspace{0.35em}
\modelwrong{A} \hfill \gt{B}

\tcbitem[title={Live Theater Operations}]
I need to tell the projection operator that the backup cable is coiled under the console. The projection operator is running live supertitles. How should I pass the message?
\vspace{\baselineskip}
\vspace{0.25em}
\textbf{A.} Walk up and tell the operator.\\
\textbf{B.} Send the note over the show chat.

\vspace{0.35em}
\modelwrong{A} \hfill \gt{B}

\tcbitem[title={Machinery Repair}]
The mower fuel-cap key blank is ready. Maintenance is at the bay to read the mower serial plate and try the blank in the actual cap. For this step, what should I do?

\vspace{0.25em}
\textbf{A.} Drive the mower to the bay.\\
\textbf{B.} Walk to the bay without the mower.

\vspace{0.35em}
\gt{A} \hfill \modelwrong{B}

\tcbitem[title={High-Stakes Event Production}]
I need to tell the lighting operator that the spare cable is under the console. The operator is running cues during a live performance. What should I do?

\vspace{0.25em}
\textbf{A.} Send a headset text/chat cue.\\
\textbf{B.} Tap and speak to the operator.

\vspace{0.35em}
\vspace{\baselineskip}
\gt{A} \hfill \modelwrong{B}

\tcbitem[title={Corporate Administrative Workflows}]
I am documenting a toll error for payroll. Payroll scans receipts, while the toll agency requires the original toll receipt for a refund. What should I submit to payroll?

\vspace{0.35em}

\textbf{A.} The original toll receipt.
\vspace{\baselineskip}\\
\textbf{B.} A receipt copy.\\

\vspace{0.35em}
\modelwrong{A} \hfill \gt{B}

\tcbitem[title={Agricultural Machinery Maintenance}]
The compact-tractor belt-slip case is open. The mechanic is at the ramp to watch the drive belt under load. For this step, what should I do?

\vspace{0.25em}
\textbf{A.} Bring the compact tractor to the ramp.\\
\textbf{B.} Walk to the ramp without the tractor.

\vspace{0.35em}
\gt{A} \hfill \modelwrong{B}

\end{tcbitemize}

\vspace{-0.5em}

\caption{
Representative examples from Real-Life Constrained QA. Each card shows a
complete two-option scenario, the ground-truth action, and an observed incorrect
model prediction. The examples illustrate key-selection failures in which a
salient associative shortcut conflicts with a prompt-grounded physical, spatial,
procedural, or medium-specific constraint.
}
\label{fig:dataset-examples}
\end{figure}

\section{Empirical Findings}
\label{sec:results}

We evaluate 2,925 Scientist QA questions across four model families and two
thinking settings, together with 500 Real-Life Constrained QA questions. All runs disable external tools, including web search.
For Scientist QA, the primary setting is the names-only prompt, where the model
sees only the two candidate names and must retrieve the decisive relation
internally. We additionally report a retrieval-relaxed profiles-in-context
control, where both candidate profiles are supplied in the prompt. Each
Scientist QA item is also paired with two closed-book probes targeting the
decisive relation. Across the primary names-only Scientist QA runs, only two
responses, both from Claude-low, fail to match either candidate after
normalization; we count these off-option responses as hallucinations.

\subsection{Retrieval-Sensitive Disambiguation Remains Difficult}
\label{sec:main_results}

Table~\ref{tab:results-profiles} compares the primary retrieval-sensitive
Scientist QA setting with the retrieval-relaxed profile control. In the
names-only prompt condition, hallucination rates range from 2.50\% to 37.23\%.
In the profiles-in-context condition, the maximum error rate is 3.38\%, and six
of the eight model settings achieve zero error. Thus, Scientist QA is not
primarily testing whether models can compare two supplied profiles; it tests
whether they can retrieve and apply the decisive relation when only the
candidate names and ambiguous description are given.

\begin{table}[t]
\centering
\footnotesize
\setlength{\tabcolsep}{2.5pt}
\begin{tabular*}{\columnwidth}{@{}p{0.16\columnwidth}p{0.22\columnwidth}rr@{}}
\toprule
\makecell[c]{Model\\ version} & 
\makecell[c]{Inference \\setting}  &
\makecell[c]{Names-only\\errors / rate} &
\makecell[c]{Profiles-in-context\\errors / rate} \\
\midrule

Claude        & low thinking      & 699 / 23.90\%  & 5 / 0.17\% \\
Claude        & high thinking     & 182 / 6.22\%   & 0 / 0.00\% \\
DeepSeek       & non-reasoning     & 1089 / 37.23\% & 99 / 3.38\% \\
DeepSeek   & reasoning         & 309 / 10.56\%  & 0 / 0.00\% \\
Gemini   & low thinking      & 73 / 2.50\%    & 0 / 0.00\% \\
Gemini   & high thinking     & 92 / 3.15\%    & 0 / 0.00\% \\
GPT                  & low thinking      & 344 / 11.76\%  & 0 / 0.00\% \\
GPT                  & high thinking     & 300 / 10.26\%  & 0 / 0.00\% \\
\bottomrule
\end{tabular*}
\caption{
Scientist QA results over 2,925 questions for frontier model versions evaluated
with external tools disabled. The names-only prompt condition corresponds to
\textit{prepend\_names}: the model receives only the two candidate names and
must retrieve the decisive relation internally. The profiles-in-context prompt
condition corresponds to \textit{prepend\_profiles}: the model receives both
candidate profiles in the prompt. Entries report the number and percentage of
errors.
}
\label{tab:results-profiles}
\end{table}

Thinking effort has model-dependent effects in the names-only setting. Higher
thinking substantially reduces errors for Claude, from 23.90\% to 6.22\%, and
for DeepSeek, from 37.23\% to 10.56\%. It modestly improves GPT, from 11.76\% to
10.26\%. Gemini-low attains the lowest observed error rate in this setting at 2.50\%,
while Gemini-high is slightly worse at 3.15\%. Because Gemini assisted in
constructing Scientist QA, direct cross-family rankings involving Gemini
should be interpreted cautiously. The within-family non-monotonicity
nevertheless shows that additional inference effort is not uniformly
beneficial.

\subsection{Direct Probes Separate Ignorance from Knowledge-Deployment Failure}
\label{sec:knowledge_decomposition}

Each Scientist QA item has two closed-book probes targeting the decisive
relation: one eliminative probe for the distractor and one compatibility probe
for the correct candidate. 

Table~\ref{tab:results-main}
conditions pairwise hallucination on these probe outcomes.\footnote{``Pairwise hall.'' is the overall names-only pairwise
error rate. ``Both probes correct'' uses all 2,925 questions as its
denominator. ``Hall.\ $\mid$ both'' and ``Hall.\ $\mid$ not both'' use,
respectively, the subset where both probes are correct and the subset where
at least one probe is incorrect. ``Known-fact hall.'' and ``Probe-absent
hall.'' use all pairwise hallucinations as their denominator and report the
fractions for which both probes are correct or both probes are incorrect,
respectively.}

\begin{table*}[t]
\centering
\small
\setlength{\tabcolsep}{4.2pt}
\begin{tabular}{llcccccc}
\toprule
Model & Mode &
\makecell[c]{Pairwise\\hall.} &
\makecell[c]{Both probes\\correct} &
\makecell[c]{Hall. $\mid$\\both} &
\makecell[c]{Hall. $\mid$\\not both} &
\makecell[c]{Known-fact\\hall.} &
\makecell[c]{Probe-absent\\hall.} \\
\midrule
Claude Sonnet 4.6   & low  & 23.90\% & 76.34\% & 18.99\% & 39.74\% & 60.66\% & 1.29\% \\
Claude Sonnet 4.6   & high &  6.22\% & 86.26\% &  2.62\% & 28.86\% & 36.26\% & 4.95\% \\
DeepSeek V3.2 Chat & low  & 37.23\% & 59.52\% & 34.46\% & 41.30\% & 55.10\% & 1.93\% \\
DeepSeek V3.2 Reasoner & high & 10.56\% & 79.28\% &  6.04\% & 27.89\% & 45.31\% & 5.50\% \\
Gemini 3.1 Pro Preview   & low  &  2.50\% & 97.13\% &  2.01\% & 19.05\% & 78.08\% & 0.00\% \\
Gemini 3.1 Pro Preview   & high &  3.15\% & 97.74\% &  2.59\% & 27.27\% & 80.43\% & 0.00\% \\
GPT-5.2      & low  & 11.76\% & 85.54\% &  7.91\% & 34.52\% & 57.56\% & 3.49\% \\
GPT-5.2      & high & 10.26\% & 87.04\% &  6.25\% & 37.20\% & 53.00\% & 2.67\% \\
\bottomrule
\end{tabular}
\caption{
Probe-conditioned results for the names-only prompt condition over 2,925
Scientist QA questions. ``Hall. $\mid$ both'' and ``Hall. $\mid$ not both''
condition pairwise hallucination on whether both probes are correct.
``Known-fact hall.'' and ``Probe-absent hall.'' report the fractions of
hallucinations where both probes are correct or both probes are wrong,
respectively.
}
\label{tab:results-main}
\end{table*}

For example, Claude-low answers both probes correctly on 76.34\% of all
questions, yet hallucinates on 18.99\% of that subset. Viewed from the error
population, 60.66\% of all Claude-low pairwise hallucinations therefore occur
despite both probes being answered correctly. Probe knowledge helps but cannot fully explain the errors. Hallucination is
much higher when not both probes are correct, yet many errors remain in
the both-probe-correct regime. Complete probe-level ignorance accounts for at most 5.50\% of errors. Thus, many are not simply missing
facts. The model can answer the relevant facts in isolation but still fail to deploy
them in pairwise disambiguation.

\subsection{Raw Fame Does Not Explain the Shortcut} \label{sec:fame_analysis_main} Because our theory emphasizes frequency-induced shortcuts, we test whether the observed Scientist QA failures reduce to a simpler fame prior. For each scientist, we define a fame score using the normalized page-view count of their Wikipedia page, the normalized length of that page, and the normalized number of external links from that page. The wrong candidate is more famous in 61.30\% of Scientist QA questions. However, hallucination does not increase in those cases. In fact, for every model setting, hallucination is lower when the wrong candidate is more famous than when it is not; for example, GPT-high drops from 13.52\% to 8.20\%, Claude-low from 34.19\% to 17.40\%, and DeepSeek-low from 41.25\% to 34.69\%. The same pattern holds from the perspective of hallucinated cases: among hallucinations, the wrong candidate is more famous only 44.64\% to 57.12\% of the time, below the dataset base rate of 61.30\%. Moreover, very famous candidates often make the task easier rather than harder. When at least one candidate is in the top 1\% by fame rank, hallucination rates are lower in all eight names-only model settings; for instance, GPT-high falls from 10.87\% to 5.11\%, and DeepSeek-low falls from 38.36\% to 27.80\%. Full fame-based analyses are in Appendix~\ref{app:fame_analysis}. These results suggest that the shortcut is not a raw preference for famous names, but a relation-specific association between entities and attributes, such as institutions, awards, roles, or fields, that can override the prompt's decisive constraint.

\subsection{Everyday Constraints Induce the Same Failure Pattern}
\label{sec:real_life_results}

Real-Life Constrained QA tests whether the same failure pattern appears outside
encyclopedic entity disambiguation. Across 500 two-option scenarios,
Claude, GPT, Gemini, and DeepSeek-chat make 81, 44, 18, and 182 errors,
respectively, corresponding to error rates of 16.20\%, 8.80\%, 3.60\%, and
36.40\% (Table~\ref{tab:real_life_results_appendix}). To test whether these failures are concentrated in a narrow construction
pattern, we further classify the final 500 questions into 12 analysis
families and 10 constraint types. No single analysis family accounts for more
than 16.6\% of the benchmark, and no constraint type accounts for more than
19.6\%. One or more models fail on 211 questions (42.2\%). Failures occur in all
10 analysis families containing at least 12 items; the two remaining
families contain only 2 and 1 items, respectively. Failures also span
9 of the 10 constraint types.

Measured as the fraction of questions missed by at least one model, the
highest rates occur for functional or diagnostic tests
(49/83, 59.0\%), vehicle-identifier verification
(44/78, 56.4\%), and rebate, warranty, or refund documentation
(38/69, 55.1\%). Failures also occur in communication-channel constraints
(6/20, 30.0\%), safety or condition constraints
(7/29, 24.1\%), and attached-label cases
(8/66, 12.1\%). These categories are semantic construction strata rather
than distinct theoretical mechanisms. Because they are unevenly sized and
were not designed for causal comparison, we treat the breakdown as
descriptive and do not claim that any category is intrinsically more
difficult. The complete breakdowns by analysis family and constraint type
are reported in Appendix~\ref{app:real_life_results}. These errors occur when a
salient shortcut action conflicts with a physical, spatial, procedural, or
medium-specific constraint stated in the prompt. Thus, inference misalignment is
not limited to biographical facts; it also appears in everyday action choice.
\paragraph{Interpretation and scope.} Taken together, these results instantiate the two complementary failure modes predicted by our framework. In Scientist QA, pairwise failures persist even when both candidate-specific probes are answered correctly, while supplying the full profiles nearly eliminates errors; this pattern is consistent with task-retrieval bias, in which behaviorally accessible facts are not deployed correctly under competition. In Real-Life Constrained QA, the decisive constraint is stated in the prompt, yet models sometimes follow a salient associative cue, consistent with key-selection bias. These observations are behavioral rather than mechanistic: because the evaluated systems and their pretraining corpora are black boxes, we do not directly observe $P(k,t|z)$, measure the relevant pretraining frequencies, or identify frequency imbalance as the unique cause of each error. TrapQA therefore diagnoses a controlled class of inference-misalignment failures rather than all hallucinations or their real-world prevalence; open-ended fabrication due to missing knowledge and failures of abstention or uncertainty calibration remain outside our empirical scope.

\section{Difficulty of Detecting Scientist QA Hallucinations}
\label{sec:scientistqa_detection}
\begin{table*}[t]
\centering
\caption{Hallucination-detection AUROC on Scientist QA using
Meta-Llama-3.1-8B-Instruct. $\dagger$ indicates access to the full scientist
profiles after answer generation; $\star$ denotes the probe-known subset.}
\label{tab:scientistqa_detection_auroc}
\small
\setlength{\tabcolsep}{4.5pt}
\begin{tabular}{llcc}
\toprule
Family & Method & Scientist QA$^\star$ & Scientist QA \\
\midrule
Representation
& Aiersilan-style \citep{aiersilan2026linearly} & .730 & .646 \\
& SAPLMA \citep{azaria2023internal} & .671 & .657 \\
& ICR Probe \citep{zhang2025icr} & .573 & .569 \\
\midrule
Uncertainty
& SelfCheckGPT \citep{manakul2023selfcheckgpt} & .593 & .580 \\
& SeSE \citep{zhao2026sese} & .560 & .552 \\
& Semantic Entropy \citep{farquhar2024semanticentropy} & .518 & .528 \\
\midrule
Evidence
& MiniCheck (unilateral) \citep{tang2024minicheck} & .717$^\dagger$ & .809$^\dagger$ \\
& MiniCheck (contrastive) \citep{tang2024minicheck} & .926$^\dagger$ & .987$^\dagger$ \\
\bottomrule
\end{tabular}
\end{table*}
The main experiments establish that models can answer the relevant facts in
isolation yet still fail on the paired Scientist QA question. We additionally
ask whether such failures can be detected from standard representation or
uncertainty signals. Table~\ref{tab:scientistqa_detection_auroc} compares
several detector families on the same Llama-3.1-8B-Instruct
generations and correctness labels. The probe-known column restricts the
evaluation to questions for which both candidate facts are independently
verified as available; the full column uses all parse-valid Scientist QA
examples.

The results show that Scientist QA errors are difficult to detect with generic
signals rather than universally undetectable. On full Scientist QA, the three
representation detectors obtain only .569--.657 AUROC and the uncertainty
detectors obtain .528--.580. These failures are often confident, so output disagreement and generic hidden-state features provide
limited separation. When complete ground-truth profiles are available after
generation, contrastive evidence verification is substantially stronger
(.987 AUROC); this is a different information-access regime and should not be
interpreted as an input-matched comparison with evidence-free detectors.
Together, these results show that Scientist QA is challenging for standard
evidence-free detectors, while reliable external evidence can make its errors
substantially easier to identify.

\section{Conclusion}
\label{sec:conclusion}

We study hallucination as a form of \emph{misaligned inference}: a model may possess the relevant facts, yet still follow a statistically dominant shortcut path that is inconsistent with the prompt's decisive constraint. We formalize this view through a latent key--task model, showing how pretraining-frequency imbalance can suppress the correct inference path and induce positive inference loss. To evaluate this perspective, we introduce TRAPQA, combining scientist
disambiguation with everyday constraint following under salient
shortcuts. In Scientist QA, pairing each primary question with
supplementary factual probes separates factual ignorance from
knowledge-deployment failure. Across frontier models, many errors
occur even when both probes are answered correctly, while providing
explicit profile context nearly eliminates the errors. These findings suggest that some observed hallucinations are not explained by knowledge storage alone, but also involve failures to deploy behaviorally accessible facts along the relevant inference path. Our findings motivate methods that go beyond factual coverage, including constraint-first reasoning that explicitly identifies the decisive condition before answering, consistency checks between final answers and independently elicited facts, and re-evaluation or abstention when the two conflict. Evaluating these interventions remains an important direction for future work.

\section*{Limitations}
\label{sec:limitations}
This research has some limitations. First, though we covered several frontier model families, our results remain limited to
the tested models: GPT-5.2, Gemini 3.1 Pro Preview, Claude Sonnet 4.6 and DeepSeek V3.2 chat/reasoning. We have explicitly reported the thinking settings and API versions, but reruns may differ as
provider-hosted systems change. Second, Gemini assisted in constructing the initial Scientist QA questions,
creating possible generator--evaluator coupling. This coupling may affect
absolute accuracy or comparative rankings, particularly Gemini's relative
performance. Our central empirical conclusion therefore rests on the presence
of probe-conditioned knowledge-deployment failures across all evaluated model
families, rather than on any direct cross-model ranking. Third, although we aim to construct questions whose answers are stable over time, some
items may still be affected by temporal drift. Scientific breakthroughs,
technological changes, or industry practice may alter what is regarded as
common sense, and scientists may later receive new honors, change
positions, or become associated with new fields as their careers continue.
Thus, future evaluations should treat the released answers as tied to the
dataset construction time and re-audit items when using TrapQA in substantially
later model evaluations.

\bibliography{reference}

\appendix
\clearpage

\section{Artifact licenses and terms}
Scientist QA uses public Wikipedia-linked scientist profiles and Wikidata QIDs.
Wikipedia text is available under CC BY-SA 4.0 unless otherwise noted, while
Wikidata structured data is released under CC0
\citep{wikipedia_contributors_2026,wikidata_contributors_2026}. Real-Life
Constrained QA uses \textsc{SWOW} only for seed selection; we cite the original
SWOW resource and do not redistribute raw SWOW participant records or
cue--response tables. The released benchmark package contains derived QA items,
labels, prompts, and saved evaluation outputs, with the final redistribution
license stated in the repository README.

\section{Artifact use and intended use}
We use existing artifacts only for research and diagnostic evaluation. Wikipedia-
and Wikidata-derived scientist information is used to construct public-profile
disambiguation questions; SWOW is used only for non-commercial seed selection,
and we do not redistribute raw SWOW data. The new TrapQA artifacts are intended
for research on hallucination, knowledge deployment, and constraint-sensitive
reasoning, not for deployment certification, individual assessment, or commercial
redistribution of source-derived data.

\section{Additional Related Work}
\label{app:related_work_extended}

This appendix expands the related-work discussion from Section~\ref{sec:related_work}, covering theoretical accounts, pipeline-stage sources of hallucination, reinforcement learning from feedback, and evaluation benchmarks.

\subsection{Mechanisms and Sources of Hallucination}
\label{app:hallucination_mechanisms}

\paragraph{Pre-LLM hallucination.}
Hallucination has long been studied in language generation. In data-to-text generation, \citet{wiseman-etal-2017-challenges} showed that neural models can produce fluent outputs that nevertheless fail to faithfully reflect the underlying records. In neural machine translation, \citet{lee2019hallucinations} analyzed hallucinations as spurious translations unrelated to the source text, a phenomenon further studied by \citet{raunak2021}. Similarly, \citet{maynez-etal-2020-faithfulness} found that neural summarization models frequently generate content that is not faithful to the source document. Although these settings differ in task formulation, they share the common problem that models may produce plausible text that is insufficiently grounded in the conditioning input.

\paragraph{Theoretical and mechanistic accounts.}
Several studies examine hallucination from the perspective of fundamental limitations. \citet{xu2025} show, in a formal learning-theoretic setting, that LLMs cannot learn all computable functions and therefore cannot completely avoid hallucination when used as general-purpose problem solvers. \citet{kalai2024} derive a statistical lower bound on hallucination for calibrated language models on certain classes of facts, suggesting that hallucination cannot be eliminated solely through better calibration. From a mechanistic perspective, \citet{sun2025llm} propose a subsequence-association framework for tracing hallucinations, arguing that hallucinations can arise when dominant hallucinatory associations outweigh faithful ones during generation. More recently, \citet{cherukuri2026} analyze hallucination through a dynamical-systems view of hidden-state trajectories, in which hallucination behavior is characterized by task-dependent latent-space basin structure.

These accounts are closely related to our work in treating hallucination as a competition between faithful and unfaithful associations. Our framework differs by focusing on prompts in which a decisive local constraint conflicts with a statistically salient shortcut. This setting allows us to study not only whether a model knows the relevant facts, but also whether it retrieves and applies the constraint-sensitive inference path required by the prompt.

\paragraph{Training, fine-tuning, and inference.}
Hallucinations can arise from multiple stages of the LLM pipeline, including pretraining, post-training, and inference. At the pretraining level, distributional imbalance can make false or misleading continuations more probable than correct ones, particularly when correct facts are rare or expressed inconsistently \citep{zhang2025measuring,liu2026pretrainrl}. More broadly, noisy, outdated, or contradictory training data may contribute to unsupported generations \citep{Ji_2023,Huang_2025}. \citet{kalai2025} further argue that modern training and evaluation procedures can reward guessing over acknowledging uncertainty, causing models to produce plausible answers even when they should abstain.

Hallucinations may also persist or emerge during fine-tuning. Several studies find that language models struggle to acquire new factual knowledge through fine-tuning \citep{gekhman2024,kang2024,lin2024flame,ren2024}. In particular, fine-tuning examples that introduce new knowledge may be learned more slowly than examples consistent with the model's pre-existing knowledge, and once learned, may increase hallucination on previously acquired facts \citep{gekhman2024}. Related work also shows that fine-tuning can differentially affect popular and unpopular factual knowledge, with models fine-tuned on more widely known facts tending to achieve higher factual accuracy than those fine-tuned on less popular facts \citep{ghosal2024}.

At inference time, hallucinations can be amplified by prompt ambiguity, decoding behavior, and reliance on memorized or frequency-biased patterns \citep{zhang2023languagemodelhallucinationssnowball,Huang_2025}. \citet{mckenna2023} identify attestation and relative-frequency biases in natural language inference as sources of hallucination-like errors, showing that models may rely on whether a hypothesis is attested in pretraining data rather than on the provided premise. \citet{berglund2024} expose a related limitation, the reversal curse, in which models trained on facts in one direction fail to reliably answer semantically equivalent queries in the reverse direction. \citet{zheng2023does} analyze ChatGPT failures in open-domain question answering and identify factuality, knowledge memorization, and knowledge recall as central sources of error. Together, these studies suggest that hallucination is not merely a matter of missing knowledge; it can also reflect failures in retrieving, comparing, or applying knowledge under the constraints of a particular prompt.

\paragraph{Reinforcement learning from feedback.}
Reinforcement learning from human feedback (RLHF) builds on preference-based reinforcement learning and human-preference fine-tuning of language models \citep{christiano2023,ziegler2020}. Its effect on hallucination remains debated. On one hand, \citet{ouyang2022} report that instruction tuning with human feedback improves truthfulness on several evaluations. On the other hand, reward models can be imperfect proxies for truthfulness, and optimizing against them may encourage reward hacking or plausible-sounding responses that satisfy human preferences without being fully faithful \citep{casper2023}. Recent work therefore proposes reinforcement-learning objectives that explicitly penalize hallucination or reward truthful abstention \citep{lin2025harnessing,wei2025truth,zhang2024r}. These approaches complement our work: rather than proposing a mitigation objective, we construct diagnostic settings that expose when models select a shortcut inference path despite having access to the relevant facts in closed-book probes.

\subsection{Hallucination Evaluation Benchmarks}
\label{app:hallucination_benchmarks}

Evaluating hallucination in language generation has attracted sustained attention, with benchmarks spanning different tasks, grounding sources, and annotation strategies. For text-only generation, early resources established task-specific foundations in summarization faithfulness, table-to-text fidelity, and open-domain factuality \citep{kryscinski2019,wang2020,pagnoni2021,fabbri2021,parikh2020totto,honovich2022,lin-etal-2022-truthfulqa,li-etal-2023-halueval,cheng2023evaluat,dale2023}. More recent work has shifted toward large-scale and increasingly automated factuality evaluation. FActScore \citep{min2023factscore} decomposes long-form outputs into atomic facts and evaluates whether each is supported by a reliable knowledge source. LongFact \citep{wei2024long} targets factuality in extended, open-domain responses, while SimpleQA \citep{wei2024measuring} provides short, fact-seeking questions with single, unambiguous answers and explicit grading of correct, incorrect, and not-attempted responses. WildHallucinations \citep{zhao2024wild} evaluates long-form factuality on real-world entity queries, including many entities outside Wikipedia coverage.

A parallel line of work examines hallucination in retrieval-augmented generation (RAG). RAGTruth \citep{niu2024ragtruth} provides human annotations of hallucinations in naturally generated RAG outputs, including word-level labels. FRAMES \citep{krishna2025} evaluates factuality, retrieval accuracy, and reasoning in end-to-end RAG scenarios, especially under multi-hop reasoning demands. In multimodal settings, researchers study hallucination in vision-language models, including object-level, relation-level, and broader multimodal hallucination detection settings \citep{guan2024,chen2024unified}. Collectively, these benchmarks reflect a progression from task-specific faithfulness evaluation toward broader, multi-task, and more automated hallucination assessment.

Our benchmark is complementary to these evaluation efforts. Existing benchmarks primarily measure whether a model's output is factual, faithful, or grounded in evidence. By contrast, Scientist QA and Real-Life Constrained QA are designed to isolate why a model fails in controlled forced-choice settings. Scientist QA tests whether models can deploy candidate-specific facts under disambiguation, while Real-Life Constrained QA tests whether models can override SWOW-derived associative cues with prompt-grounded physical, spatial, procedural, or medium-specific constraints \citep{de2019small}. This design allows us to distinguish simple ignorance from knowledge-deployment failures in which the relevant information is available to the model but not used in the task-relevant inference path.

\section{Scientist QA Construction Details}
\label{app:benchmark_details}

This appendix gives the full Scientist QA construction pipeline: profile collection, name-removed profile linearization, hard-pair mining, question generation, and filtering. Figure~\ref{fig:pipeline_overview} summarizes the construction pipeline.

\subsection{Scientist Profiles}
\label{app:profile_processing}

We collect 9,090 scientists with dedicated Wikipedia pages. Each scientist is represented as a structured profile containing attributes such as occupation, field, notable work, awards, education, and a Wikidata QID. The QID is used only for bookkeeping, deduplication, and linking candidates across processing stages. Appendix~\ref{app:long_profile_example} shows an example profile.

For pair mining, we remove the scientist's name from each profile and linearize the remaining attributes into a short paragraph, which we call the \emph{name-removed profile}. This prevents the embedding model from matching scientists by name while preserving semantic information from their profile attributes. Appendix~\ref{app:embedding_ready_profile} shows an example.

\subsection{Hard-Pair Mining}
\label{app:hard_pair_mining}

Let $\mathbf e_A$ and $\mathbf e_B$ denote the embeddings of the name-removed profiles of scientists $A$ and $B$, obtained using \texttt{text-embedding-3-small} \citep{openai2024embedding3small}. Let $\mathrm{TC}(A)$ be the number of retained attribute fields for scientist $A$, and let $\lambda$ be the median tag count across all scientists; in our data, $\lambda=7$. We score each pair by
\begin{align}
s(A,B)=
&\frac{\mathbf e_A^\top \mathbf e_B}{\|\mathbf e_A\|\,\|\mathbf e_B\|} \cdot \notag \\
&\frac{\min(\mathrm{TC}(A),\mathrm{TC}(B))}
{\min(\mathrm{TC}(A),\mathrm{TC}(B))+\lambda}\label{eq:pair_similarity_appendix}
\end{align}
The cosine term measures semantic similarity, while the penalty term downweights pairs whose similarity may be driven by sparse metadata. We rank all scientist pairs by Equation~\ref{eq:pair_similarity_appendix} and retain the top $0.01\%$ highest-scoring pairs, yielding 2,958 candidate pairs.

\subsection{Question Generation and Filtering}
\label{app:question_generation_details}

For each candidate pair $(A,B)$, we prompt Gemini to generate a third-person biographical paragraph broadly compatible with both scientists, followed by a single decisive constraint that rules out exactly one candidate. Each question ends with \emph{``Who is this person?''} We then use ChatGPT to filter malformed or unverifiable outputs, including cases where the decisive constraint does not distinguish the pair, contradicts the candidate profiles, or cannot be verified against the paired candidates. After filtering and excluding the 33 invalid items identified in the final proofreading pass, the final evaluation set contains 2,925 questions.

Each retained question is evaluated in two variants:
\begin{enumerate}[leftmargin=*]
    \item \textit{prepend\_names}, which prepends only the two candidate names;
    \item \textit{prepend\_profiles}, which prepends the full profiles of both candidates.
\end{enumerate}
Each question is also paired with two supplementary probe questions derived from the decisive constraint, one for each candidate.

\begin{figure*}[t]
    \centering
    \includegraphics[width=\linewidth]{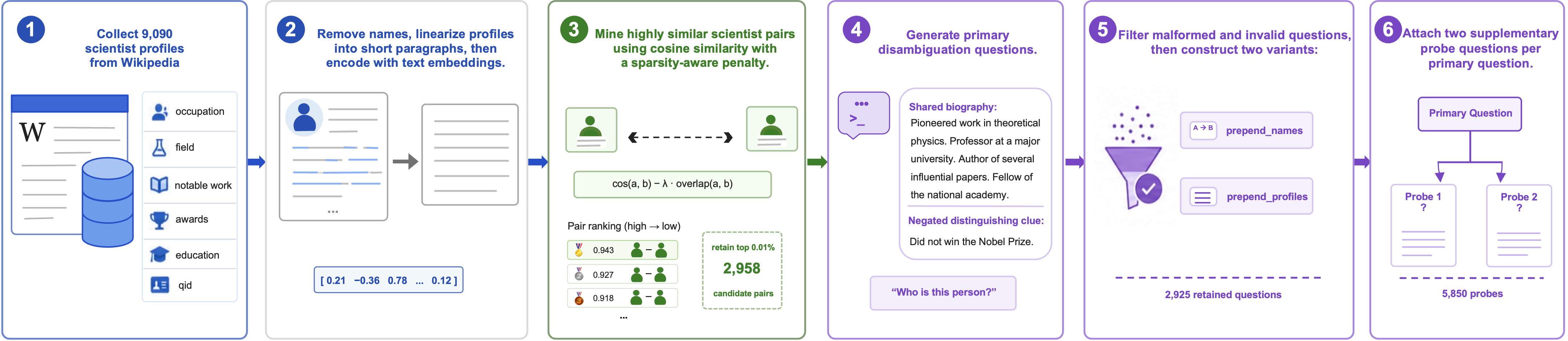}
    \caption{
    Overview of the Scientist QA construction pipeline. Starting from Wikipedia-linked scientist profiles, we construct highly confusable scientist pairs, generate pairwise disambiguation questions, and attach two supplementary probes to each primary question.
    }
    \label{fig:pipeline_overview}
\end{figure*}

\section{Implementation Details}
\label{app:implementation_details}

This appendix specifies the prompt formats, answer parsing rules, and failure handling used in evaluation.

\subsection{Model versions}

For reproducibility, Table~\ref{tab:model_versions} reports the model identifiers and inference settings used in evaluation. Exact access dates, provider-side parameters, and run identifiers should be preserved with the released run logs.

\begin{table*}[t]
\centering
\caption{Model identifiers and inference settings used in the evaluation. Low- and high-thinking settings are provider-specific controls; for DeepSeek, the two columns correspond to non-reasoning and reasoning model aliases.}
\label{tab:model_versions}
\begin{tabular}{llll}
\toprule
Family & Non-reasoning API ID & Reasoning API ID \\
\midrule
GPT-5.2 & \texttt{gpt-5.2-2025-12-11} & \texttt{gpt-5.2-2025-12-11}\\
Claude Sonnet 4.6 & \texttt{claude-sonnet-4-6} & \texttt{claude-sonnet-4-6} \\
Gemini 3.1 Pro Preview & \texttt{gemini-3.1-pro-preview} & \texttt{gemini-3.1-pro-preview} \\
DeepSeek V3.2& \texttt{deepseek-chat} & \texttt{deepseek-reasoner} \\
\bottomrule
\end{tabular}
\end{table*}

\subsection{Prompt Templates}
\label{app:prompt_templates}

We evaluate each benchmark item in separate conversations: one for the primary question and one for each supplementary probe.

\paragraph{Names-only prompt (\textit{prepend\_names}).}
\begin{quote}
\texttt{Choose one of the following two options as the answer to the question below:}\\
\texttt{1. }$A$\\
\texttt{2. }$B$\\
\texttt{Question:}\\
$\textit{question}_i$
\end{quote}
Here $A$ and $B$ are the two candidate scientists. Their order is randomized across items.

\paragraph{Profiles-in-context prompt (\textit{prepend\_profiles}).}
\begin{quote}
\texttt{Given two profiles of two persons:}\\
$\textit{profile}_A$\\
$\textit{profile}_B$\\
\texttt{Choose exactly one profile from the two, and output the name of the person as the answer to the following question:}\\
$\textit{question}_i$
\end{quote}

\paragraph{Supplementary probes.}
Each supplementary probe is asked independently as a binary factual question about one candidate and the decisive relation. For example:
\begin{quote}
\texttt{Did Albert Einstein receive the Nobel Prize in Physics?}
\end{quote}

\subsection{Answer Matching and Failure Handling}
\label{app:answer_matching}

For primary questions, the model is instructed to output exactly one of the two candidate names. We normalize whitespace, capitalization, and minor formatting differences before matching. If the normalized response matches the correct candidate, it is counted as correct; if it matches the distractor, it is counted as a hallucination. If the response matches neither candidate, it is also counted as a hallucination. Across the primary \textit{prepend\_names} Scientist QA experiments, only two unmatched primary-question responses remain after normalization, both from Claude-low.

For supplementary probes, binary answers are normalized to true/false labels. Each primary-question outcome is then paired with its two probe outcomes to determine whether the model \emph{knows both}, \emph{knows one}, or \emph{knows neither} of the relevant probe facts.
\clearpage
\section{Additional Examples}
\label{app:examples}

\subsection{Profile Example}
\label{app:long_profile_example}
Figure ~\ref{fig:structured-profile} provides a complete profile of Wolfgang Pauli.
\begin{figure*}[t]
\centering
\begin{tcolorbox}[
  width=\textwidth,
  colback=white,
  colframe=gray,
  title=\textbf{Example structured profile: Wolfgang Pauli},
  title filled=false
]
\small
\begin{verbatim}
{
  "Wolfgang Pauli": {
    "occupation": [
      "theoretical physicist",
      "university teacher",
      "chemist",
      "physicist"
    ],
    "award_received": [
      "Nobel Prize in Physics",
      "Max Planck Medal",
      "Lorentz Medal",
      "Foreign Member of the Royal Society",
      "honorary doctor of the University of Vienna"
    ],
    "field": [
      "quantum mechanics",
      "particle physics"
    ],
    "notable_work": [
      "Pauli exclusion principle",
      "Pauli matrices",
      "Pauli equation"
    ],
    "education": [
      "Ludwig-Maximilians-Universität München",
      "Bundesgymnasium Döbling"
    ],
    "qid": "Q65989"
  }
}
\end{verbatim}
\end{tcolorbox}
\caption{Example structured scientist profile used in Scientist QA.}
\label{fig:structured-profile}
\end{figure*}

\subsection{Name-Removed Profile Example}
\label{app:embedding_ready_profile}

\begingroup
\begin{tcolorbox}[colback=white, colframe=gray, title=\textbf{Example name-removed profile: Albert Einstein}, title filled=false]
occupation: inventor, mathematician, philosopher of science, physicist, professor, science writer, theoretical physicist, university teacher; field: theoretical physics; notable work: general relativity, mass--energy equivalence, photoelectric effect, quantum mechanics, special relativity, theory of Brownian motion, theory of relativity; awards: Copley Medal, Foreign Member of the Royal Society, Franklin Medal, Max Planck Medal, Nobel Prize in Physics, Pour le Mérite; education: ETH Zurich, Luitpold-Gymnasium, University of Zurich; positions: professor.
\end{tcolorbox}
\endgroup

\subsection{Question Example}
\label{app:question_example}
Figure~\ref{fig:prepend-profiles-example} shows the profiles-in-context prompt variant.
\begingroup
\begin{tcolorbox}[colback=white, colframe=gray, title=\textbf{Primary question example}, title filled=false]
This prominent theoretical physicist, mathematician, and university teacher made significant contributions to science. In recognition of their work, they delivered the Josiah Willard Gibbs Lectureship and were elected a Foreign Member of the Royal Society. However, this scientist never received the Nobel Prize in Physics. Who is this person?
\end{tcolorbox}

\begin{tcolorbox}[colback=white, colframe=gray, title=\textbf{Paired supplementary probes}, title filled=false]
Did Albert Einstein receive the Nobel Prize in Physics?\\
Did Edward Witten receive the Nobel Prize in Physics?
\end{tcolorbox}

\begin{tcolorbox}[colback=white, colframe=gray, title=\textbf{\textit{prepend\_names} variant}, title filled=false]
Choose one of the following two options as the answer to the question below:\\
1. Edward Witten\\
2. Albert Einstein\\
Question:\\
This prominent theoretical physicist, mathematician, and university teacher made significant contributions to science. In recognition of their work, they delivered the Josiah Willard Gibbs Lectureship and were elected a Foreign Member of the Royal Society. However, this scientist never received the Nobel Prize in Physics. Who is this person?
\end{tcolorbox}
\endgroup

\begin{figure*}[t]
\centering
\begin{tcolorbox}[
  width=\textwidth,
  colback=white,
  colframe=gray,
  title=\textbf{\textit{prepend\_profiles} variant},
  title filled=false
]
\small
Given two profiles of two persons:
\begin{lstlisting}[
  basicstyle=\ttfamily\footnotesize,
  breaklines=true,
  columns=fullflexible
]
name: Edward Witten
occupation: mathematician; physicist; university teacher; theoretical physicist
award_received: Fields Medal; MacArthur Fellowship; Isaac Newton Medal; ...
field: physics; mathematical physics; string theory
education: Princeton University; University of Wisconsin--Madison; ...

name: Albert Einstein
occupation: theoretical physicist; philosopher of science; science writer; ...
award_received: Nobel Prize in Physics; Copley Medal; Franklin Medal; ...
field: theoretical physics
notable_work: general relativity; special relativity; photoelectric effect; ...
education: ETH Zurich; University of Zurich; ...
\end{lstlisting}

Choose exactly one profile from the two, and output the name of the person as the answer to the following question:

\medskip
\noindent
This prominent theoretical physicist, mathematician, and university teacher made significant contributions to science. In recognition of their work, they delivered the Josiah Willard Gibbs Lectureship and were elected a Foreign Member of the Royal Society. However, this scientist never received the Nobel Prize in Physics. Who is this person?
\end{tcolorbox}
\caption{Example \textit{prepend\_profiles} prompt variant. The ellipses in the profile example indicate omitted attributes for readability.}
\label{fig:prepend-profiles-example}
\end{figure*}

\FloatBarrier
\clearpage
\section{Real-Life Constrained QA Construction Details}
\label{app:real_life_constrained_qa_details}

This appendix describes Real-Life Constrained QA, a collection of realistic two-option questions in which a locally plausible shortcut conflicts with a physical, spatial, procedural, or medium-specific constraint. Unlike Scientist QA, which tests entity disambiguation among highly similar scientists, this component targets shortcut-driven failures in everyday scenarios. Each item presents a short first-person situation and two candidate actions or media. One option is superficially attractive because it matches a strong prior association, while the other is correct because it satisfies the constraint implied by the scenario. The final collection contains 500 questions covering 13 aspects of daily life.

\subsection{Association Mining from \textsc{SWOW}}
\label{app:csqa_seed_mining}

We begin from \textsc{SWOW} \citep{de2019small}, a large-scale psycholinguistic resource of human word associations. For each cue word, we use high-probability first responses as candidate shortcut associations. We lightly normalize and filter these associations by lowercasing, lemmatizing, merging obvious duplicates, and removing generic or noisy responses. The result is a cleaned bank of human-salient cue--response pairs suitable for question generation.
We used several preprocessing packages for SWOW seed preprocessing. We use spaCy with the \texttt{en\_core\_web\_sm}
English pipeline for tokenization, POS/stopword checks, and lemmatization; NLTK
WordNet for coarse lexical-type labels from synsets and lexnames; and
\texttt{wordfreq} Zipf frequencies to filter overly common or rare responses.
We disable the spaCy parser for this preprocessing step and use heuristic
frequency thresholds of \texttt{high\_zipf=6.5} and \texttt{low\_zipf=1.0} in
the first-pass filter.

\subsection{Template Families and Seed Selection}
\label{app:csqa_template_families}

We organize cleaned associations into eight seed template families corresponding to recurring hidden-constraint patterns. Examples include \texttt{vehicle\_required}, where the task requires bringing a vehicle rather than merely reaching a location; \texttt{delivery\_medium}, where a physical item cannot be replaced by a digital surrogate; \texttt{recording\_medium}, where the correct action depends on the required recording modality; and \texttt{tool\_required}, where a specific tool is necessary for task completion.

For each seed, we annotate structured metadata, including the scenario role, latent constraint type, and intended shapes of the correct and shortcut options. We prioritize seeds whose associations are concrete, whose constraints are easy to instantiate in everyday settings, and whose shortcut options are plausible without being absurd. We also cap overrepresented lemmas within each family to maintain scenario diversity.

\subsection{Generation, Filtering, and Augmentation}
\label{app:csqa_generation_filtering}

For each selected seed, we use GPT to augment seed questions into multiple
first-person scenarios following the corresponding family template. Each
generated item must be realistic, self-contained, and unambiguous: the incorrect
option should be a plausible shortcut, while the correct option should be
determined by a recoverable constraint in the scenario. Claude is then used to
proofread the resulting questions for ambiguity, plausibility, and constraint
validity. We manually filter malformed or weak items, including cases where both
options are arguably valid, the constraint is too explicit, the scenario depends
on niche expertise, or the shortcut option is implausible.

\subsection{Benchmark Format}
\label{app:csqa_format}

Each Real-Life Constrained QA item consists of a short scenario, two candidate options and a gold label.

\section{Extended Empirical Results}
\label{app:extended_results}

This appendix provides the full empirical breakdowns supporting Section~\ref{sec:results}. Unless otherwise stated, all tables refer to the retrieval-sensitive \textit{prepend\_names} condition over 2,925 Scientist QA questions.

\subsection{Full Probe-State Breakdowns}
\label{app:probe_breakdowns}

For each pairwise question, we use two closed-book supplementary probes targeting the decisive relation. We group examples into three probe-defined knowledge states:
\begin{itemize}[leftmargin=*]
    \item \textbf{Knows both:} both supplementary probes are answered correctly;
    \item \textbf{Knows one:} exactly one supplementary probe is answered correctly;
    \item \textbf{Knows neither:} both supplementary probes are answered incorrectly.
\end{itemize}
Table~\ref{tab:probe_breakdown_appendix} reports correct and incorrect pairwise outcomes within each state. Correct answers are typically concentrated in the \emph{knows-both} bucket, while incorrect answers shift toward the \emph{knows-one} and \emph{knows-neither} buckets. However, the \emph{knows-both} rows still contain nonzero error rates, showing that correct probe-level knowledge does not guarantee correct comparative deployment.

\begin{table*}[t]
\centering
\scriptsize
\setlength{\tabcolsep}{3.2pt}
\begin{tabular}{llccc ccc ccc}
\toprule
Model & Mode &
\makecell[c]{Knows \\both\\correct} &
\makecell[c]{Knows \\both\\wrong} &
\makecell[c]{Wrong \\rate} &
\makecell[c]{Knows \\one\\correct} &
\makecell[c]{Knows \\one\\wrong} &
\makecell[c]{Wrong \\rate} &
\makecell[c]{Knows \\neither\\correct} &
\makecell[c]{Knows \\neither\\wrong} &
\makecell[c]{Wrong \\rate} \\
\midrule
Claude Sonnet 4.6   & high & 2457 & 66  & 2.62\%  & 279 & 107 & 27.72\% & 7  & 9  & 56.25\% \\
Claude Sonnet 4.6   & low  & 1809 & 424 & 18.99\% & 401 & 266 & 39.88\% & 16 & 9  & 36.00\% \\
DeepSeek V3.2 Reasoner & high & 2179 & 140 & 6.04\%  & 410 & 152 & 27.05\% & 27 & 17 & 38.64\% \\
DeepSeek V3.2 Chat & low  & 1141 & 600 & 34.46\% & 665 & 468 & 41.31\% & 30 & 21 & 41.18\% \\
Gemini 3.1 Pro Preview   & high & 2785 & 74  & 2.59\%  & 48  & 18  & 27.27\% & 0  & 0  & -- \\
Gemini 3.1 Pro Preview   & low  & 2784 & 57  & 2.01\%  & 68  & 16  & 19.05\% & 0  & 0  & -- \\
GPT-5.2      & high & 2387 & 159 & 6.25\%  & 229 & 133 & 36.74\% & 9  & 8  & 47.06\% \\
GPT-5.2      & low  & 2304 & 198 & 7.91\%  & 260 & 134 & 34.01\% & 17 & 12 & 41.38\% \\
\bottomrule
\end{tabular}
\caption{
Probe-conditioned breakdown of primary-question outcomes in the names-only Scientist QA condition. Each bucket is defined by the number of supplementary probes answered correctly. ``Correct'' and ``wrong'' count pairwise disambiguation outcomes within that bucket, and ``Wrong rate'' is computed within the corresponding bucket.
}
\label{tab:probe_breakdown_appendix}
\end{table*}

\subsection{Eliminative-Probe Asymmetry}
\label{app:probe_asymmetry}

The two supplementary probes play different diagnostic roles. One tests the fact that should eliminate the distractor; the other tests the compatibility of the correct candidate with the decisive constraint. Table~\ref{tab:probe_asymmetry_appendix} focuses on one-probe-correct cases. Across all eight model settings, hallucination is higher when the model misses the eliminative probe than when it misses the compatibility probe. This asymmetry supports the latent key--task account in Section~\ref{sec:freq}: the decisive relation is often not merely a fact about the correct candidate, but the fact that suppresses the shortcut candidate. When this eliminative fact is not retrieved, the high-salience candidate remains available as a plausible continuation.

\begin{table*}[t]
\centering
\small
\setlength{\tabcolsep}{4.2pt}
\begin{tabular}{llccccc}
\toprule
Model & Mode &
\makecell[c]{$n$ missing\\elim.} &
\makecell[c]{Hall. when\\elim. missed} &
\makecell[c]{$n$ missing\\compat.} &
\makecell[c]{Hall. when\\compat. missed} &
\makecell[c]{Gap} \\
\midrule
Claude Sonnet 4.6   & high & 128 & 28.91\% & 258 & 27.13\% & 1.77 \\
Claude Sonnet 4.6   & low  & 489 & 41.72\% & 178 & 34.83\% & 6.89 \\
DeepSeek V3.2 Reasoner & high & 272 & 29.04\% & 290 & 25.17\% & 3.87 \\
DeepSeek V3.2 Chat & low  & 298 & 46.64\% & 835 & 39.40\% & 7.24 \\
Gemini 3.1 Pro Preview   & high & 40  & 40.00\% & 26  & 7.69\%  & 32.31 \\
Gemini 3.1 Pro Preview   & low  & 35  & 31.43\% & 49  & 10.20\% & 21.22 \\
GPT-5.2      & high & 126 & 46.03\% & 236 & 31.78\% & 14.25 \\
GPT-5.2      & low  & 153 & 37.91\% & 241 & 31.54\% & 6.37 \\
\bottomrule
\end{tabular}
\caption{
Hallucination rates in one-probe-correct cases for the names-only Scientist QA condition. ``Elim.'' denotes the probe whose correct answer eliminates the distractor; ``compat.'' denotes the probe whose correct answer confirms the correct candidate's compatibility with the decisive constraint. The gap is the difference between the two hallucination rates.
}
\label{tab:probe_asymmetry_appendix}
\end{table*}

\subsection{Consensus Failures}
\label{app:consensus_failures}

To distinguish idiosyncratic model errors from shared shortcut directions, we identify questions missed by multiple model settings. Of the 2,925 questions, 1,489 are missed by at least one model setting, and 10 are missed by all eight settings. Table~\ref{tab:consensus_failures_appendix} lists these all-setting consensus failures. They concentrate on high-frequency biographical relation families such as education, awards, professional roles, and offices. In these cases, the distractor satisfies a salient affirmative association, while the correct answer is determined by an explicit incompatibility or non-possession constraint. These examples support the common-shortcut assumption in Theorem~\ref{thm:posterior}: different model families can be biased toward the same incorrect answer when a dominant association conflicts with the prompt's decisive constraint.

\begin{table*}[t]
\centering
\small
\setlength{\tabcolsep}{4pt}
\begin{tabular}{lll l}
\toprule
Question ID & Correct candidate & Distractor & Decisive relation family \\
\midrule
question\_0214 & Klaus von Klitzing & Rudolf Mössbauer & {Education / institution} \\
question\_0596 & Glenn T. Seaborg & Mildred Dresselhaus & {Education / institution} \\
question\_0797 & Jennifer Doudna & Frances Arnold & {Award / honor} \\
question\_1092 & Fritz Lipmann & Otto Heinrich Warburg & {Education / institution} \\
question\_1161 & Norman Foster Ramsey, Jr. & Carl Wieman & {Award / honor} \\
question\_1517 & Joseph-Louis Lagrange & François Arago & {Political office / role} \\
question\_1772 & Alexander R. Todd, Baron Todd & Svante Arrhenius & {Occupation / role} \\
question\_1981 & Harold Clayton Urey & Mildred Dresselhaus & {Education / institution} \\
question\_2183 & Steven Weinberg & Leon Max Lederman & {Award / honor} \\
question\_2370 & Robert Aumann & Gérard Debreu & {Education / institution} \\
\bottomrule
\end{tabular}
\caption{
The 10 Scientist QA questions missed by all eight model settings. Question IDs and candidate names are produced by the analysis notebook; relation-family labels are manual annotations based on the decisive constraint.
}
\label{tab:consensus_failures_appendix}
\end{table*}

\subsection{Probe-Underdetermined Cases}
\label{app:probe_underdetermined}

We refer to the one-probe-correct subset as \textbf{probe-underdetermined} from the model's local probe state. Since the two gold probe labels are complementary, answering exactly one probe correctly is equivalent to predicting the same value for both probes, either both true or both false. In this regime, the two probe predictions alone do not determine the correct pairwise answer for the model. Table~\ref{tab:probe_underdetermined_appendix} reports the size and behavior of this subset.

\begin{table*}[t]
\centering
\small
\setlength{\tabcolsep}{3.6pt}
\begin{tabular}{llcccccc}
\toprule
Model & Mode &
\makecell[c]{$n$ probe-\\underdet.} &
\makecell[c]{Pct. of\\questions} &
\makecell[c]{Accuracy} &
\makecell[c]{Hall. rate} &
\makecell[c]{Both predicted\\false} &
\makecell[c]{Both predicted\\true} \\
\midrule
Claude Sonnet 4.6   & high & 386  & 13.20\% & 72.28\% & 27.72\% & 33.16\% & 66.84\% \\
Claude Sonnet 4.6   & low  & 667  & 22.80\% & 60.12\% & 39.88\% & 73.31\% & 26.69\% \\
DeepSeek V3.2 Reasoner & high & 562  & 19.21\% & 72.95\% & 27.05\% & 48.40\% & 51.60\% \\
DeepSeek V3.2 Chat & low  & 1133 & 38.74\% & 58.69\% & 41.31\% & 26.30\% & 73.70\% \\
Gemini 3.1 Pro Preview   & high & 66   & 2.26\%  & 72.73\% & 27.27\% & 60.61\% & 39.39\% \\
Gemini 3.1 Pro Preview   & low  & 84   & 2.87\%  & 80.95\% & 19.05\% & 41.67\% & 58.33\% \\
GPT-5.2      & high & 362  & 12.38\% & 63.26\% & 36.74\% & 34.81\% & 65.19\% \\
GPT-5.2      & low  & 394  & 13.47\% & 65.99\% & 34.01\% & 38.83\% & 61.17\% \\
\bottomrule
\end{tabular}
\caption{
Results on probe-underdetermined cases in the names-only Scientist QA condition. ``Pct. of questions'' uses 2,925 as the denominator. ``Both predicted false'' and ``Both predicted true'' describe the model's two probe predictions within this subset. These cases are common for weaker settings, especially DeepSeek-low and Claude-low, but rare for Gemini.
}
\label{tab:probe_underdetermined_appendix}
\end{table*}

Table~\ref{tab:probe_underdetermined_fame_appendix} further shows that behavior in this regime is not explained by simply choosing the more famous scientist. In several settings, the model chooses the more famous candidate less than half the time.

\begin{table*}[t]
\centering
\small
\setlength{\tabcolsep}{4.5pt}
\begin{tabular}{llccc}
\toprule
Model & Mode &
\makecell[c]{$n$ non-tie} &
\makecell[c]{Chooses more\\famous} &
\makecell[c]{$p$ vs. 50\%} \\
\midrule
Claude Sonnet 4.6   & high & 386  & 46.89\% & 0.242 \\
Claude Sonnet 4.6   & low  & 667  & 44.68\% & 0.007 \\
DeepSeek V3.2 Reasoner & high & 562  & 49.47\% & 0.833 \\
DeepSeek V3.2 Chat & low  & 1133 & 46.07\% & 0.009 \\
Gemini 3.1 Pro Preview   & high & 66   & 50.00\% & 1.000 \\
Gemini 3.1 Pro Preview   & low  & 84   & 46.43\% & 0.586 \\
GPT-5.2      & high & 362  & 45.86\% & 0.127 \\
GPT-5.2      & low  & 394  & 46.70\% & 0.208 \\
\bottomrule
\end{tabular}
\caption{
Choice of the more famous candidate within probe-underdetermined, non-tie cases. ``Chooses more famous'' is the fraction of such cases in which the pairwise answer is the candidate with the higher fame score. The binomial test compares the observed rate to a 50\% baseline.
}
\label{tab:probe_underdetermined_fame_appendix}
\end{table*}

\subsection{Fame-Based Analyses}
\label{app:fame_analysis}

We examine whether hallucinations can be explained by a simple prior toward the more famous scientist. For each scientist $s$, we define
\begin{align*}
\mathrm{Fame}(s)=
\frac{1}{3}[
&\mathrm{norm}(\mathrm{pageLength}_s)\ +\\
&\mathrm{norm}(\mathrm{pageViews})\ +\\
&\mathrm{norm}(\mathrm{externalLinks}_s)
]
\end{align*}
where $\mathrm{norm}(\cdot)$ denotes corpus-level normalization across scientists. The fame rank is induced by this score.

We use the 2020-01-01$\ \text{\textasciitilde}\ $2025-12-31 calendar-year window for the page view count because it is the most recent complete multi-year window before our evaluation period \citep{wikimedia_pageviews_2026}. This window balances recency with robustness to short-term spikes in public attention and avoids using page-view data generated after the benchmark evaluation.

These analyses serve as negative controls. The wrong candidate is more famous in 61.30\% of benchmark questions, but among hallucinated cases this fraction is lower, ranging from 44.64\% to 57.12\% depending on the model setting. Table~\ref{tab:fame_direction_appendix} shows that hallucination rates are lower, not higher, when the incorrect candidate is more famous.

\begin{table*}[t]
\centering
\small
\setlength{\tabcolsep}{4.5pt}
\begin{tabular}{llccc}
\toprule
Model & Mode &
\makecell[c]{Hall. when wrong\\not more famous} &
\makecell[c]{Hall. when wrong\\more famous} &
\makecell[c]{Wrong more famous\\among hallucinations} \\
\midrule
Claude Sonnet 4.6   & high & 8.83\%  & 4.57\%  & 45.05\% \\
Claude Sonnet 4.6   & low  & 34.19\% & 17.40\% & 44.64\% \\
DeepSeek V3.2 Reasoner & high & 13.16\% & 8.92\%  & 51.78\% \\
DeepSeek V3.2 Chat & low  & 41.25\% & 34.69\% & 57.12\% \\
Gemini 3.1 Pro Preview   & high & 4.24\%  & 2.45\%  & 47.83\% \\
Gemini 3.1 Pro Preview   & low  & 3.53\%  & 1.84\%  & 45.21\% \\
GPT-5.2      & high & 13.52\% & 8.20\%  & 49.00\% \\
GPT-5.2      & low  & 15.37\% & 9.48\%  & 49.42\% \\
\bottomrule
\end{tabular}
\caption{
Fame-direction negative control for the names-only Scientist QA condition. The first two numeric columns condition on whether the wrong candidate is more famous by fame score. The final column reports, among hallucinated cases, the fraction in which the wrong candidate is more famous. Hallucination rates are consistently lower when the wrong candidate is more famous, indicating that the observed shortcut is not a simple more-famous-name prior.
}
\label{tab:fame_direction_appendix}
\end{table*}

\subsection{Confidence Diagnostics}
\label{app:confidence_diagnostics}

Accuracy and confidence are imperfect certificates of faithful reasoning. First, a model can sometimes answer the pairwise question correctly without answering both probes correctly; for example, only 62.15\% of DeepSeek-low's correct pairwise answers occur in the both-probe-correct regime. This suggests that some correct answers may be supported by shortcuts that happen to point to the correct candidate.

Second, self-reported confidence does not reliably separate correct from hallucinated answers across model families. Table~\ref{tab:confidence_appendix} reports confidence for correct and hallucinated pairwise answers. Hallucinated answers are usually less confident than correct answers, but the absolute confidence remains high in many settings. DeepSeek-low is especially poorly separated: hallucinated and correct answers have nearly identical mean confidence.

\begin{table*}[t]
\centering
\small
\setlength{\tabcolsep}{4.5pt}
\begin{tabular}{llccc}
\toprule
Model & Mode &
\makecell[c]{Mean conf.\\correct} &
\makecell[c]{Mean conf.\\hallucinated} &
\makecell[c]{Gap} \\
\midrule
Claude Sonnet 4.6   & high & 88.06 & 74.19 & -13.87 \\
Claude Sonnet 4.6   & low  & 72.16 & 62.46 & -9.70 \\
DeepSeek V3.2 Reasoner & high & 92.79 & 86.48 & -6.31 \\
DeepSeek V3.2 Chat & low  & 84.89 & 84.93 & 0.04 \\
Gemini 3.1 Pro Preview   & high & 99.23 & 92.55 & -6.68 \\
Gemini 3.1 Pro Preview   & low  & 97.38 & 76.23 & -21.15 \\
GPT-5.2      & high & 91.10 & 81.35 & -9.76 \\
GPT-5.2      & low  & 90.81 & 83.55 & -7.27 \\
\bottomrule
\end{tabular}
\caption{
Mean self-reported confidence for correct and hallucinated pairwise answers in the names-only Scientist QA condition. The gap is hallucinated confidence minus correct confidence. Confidence separates correct and incorrect answers for some models, but not reliably across model families.
}
\label{tab:confidence_appendix}
\end{table*}

\subsection{Real-Life Constrained QA Results}
\label{app:real_life_results}

Real-Life Constrained QA contains 500 two-option scenarios covering 13 aspects of daily life. Table~\ref{tab:real_life_results_appendix} reports the final error counts and rates for the evaluated models.

\begin{table*}[t]
\centering
\small
\setlength{\tabcolsep}{6pt}
\begin{tabular}{lc}
\toprule
Model & Errors / rate \\
\midrule
Claude Sonnet 4.6 & 81 / 16.20\% \\
DeepSeek-chat & 182 / 36.40\% \\
GPT-5.2 & 44 / 8.80\% \\
Gemini 3.1 Pro Preview & 18 / 3.60\%\\
\bottomrule
\end{tabular}
\caption{
Real-Life Constrained QA results over 500 questions covering 13 aspects of daily life. Entries report the number and percentage of incorrect shortcut selections.
}
\label{tab:real_life_results_appendix}
\end{table*}

\paragraph{Analysis-family breakdown.}
For descriptive analysis, we group the final 500 questions into 12 semantic
analysis families. These are post-construction groupings of the final
questions and are distinct from the eight seed template families used during
generation in Appendix~\ref{app:csqa_template_families}. Table~\ref{tab:real_life_family_breakdown}
reports the size of each family and the number of questions answered
incorrectly by at least one evaluated model. At least one model makes an
error on 211 questions (42.2\%). Errors occur in all 10 families containing
at least 12 items; the two remaining families contain only 2 and 1 items,
respectively. Because these strata are unevenly sized and were not designed
for causal comparison, we treat the rates as descriptive and do not interpret
them as estimates of intrinsic family difficulty.

\begin{table*}[t]
    \centering
    \footnotesize
    \setlength{\tabcolsep}{3pt}
    \renewcommand{\arraystretch}{0.96}
    \begin{tabular}{@{}p{0.61\columnwidth}rr@{}}
        \toprule
        Analysis family & $n$ & Missed \\
        \midrule
        Functional or diagnostic test
            & 83 & 49 (59.0\%) \\
        Vehicle identifier verification
            & 78 & 44 (56.4\%) \\
        Physical scan, photo, or measurement
            & 69 & 33 (47.8\%) \\
        Rebate, warranty, or refund documentation
            & 69 & 38 (55.1\%) \\
        Attached-label or container scan
            & 66 & 8 (12.1\%) \\
        Vehicle positioning, handoff, or loading
            & 42 & 10 (23.8\%) \\
        Original/copy document routing
            & 29 & 15 (51.7\%) \\
        Safety or condition rule
            & 29 & 7 (24.1\%) \\
        Communication-channel choice
            & 20 & 6 (30.0\%) \\
        Object presence, repair, or adjustment
            & 12 & 1 (8.3\%) \\
        Access-card or badge use
            & 2 & 0 (0.0\%) \\
        Other physical dependency
            & 1 & 0 (0.0\%) \\
        \midrule
        \textbf{All questions}
            & \textbf{500} & \textbf{211 (42.2\%)} \\
        \bottomrule
    \end{tabular}
    \caption{Breakdown of Real-Life Constrained QA by post-construction
    analysis family. ``Missed'' reports the number of questions answered
    incorrectly by at least one evaluated model, followed by the corresponding
    within-family percentage. Family sizes are uneven, so the rates are
    intended as descriptive coverage statistics rather than comparative
    estimates of category difficulty.}
    \label{tab:real_life_family_breakdown}
\end{table*}

\paragraph{Constraint-type breakdown.}
We additionally group the 500 questions by the 10 annotated constraint
types. Table~\ref{tab:real_life_constraint_breakdown} reports the size of
each type and the number of questions answered incorrectly by at least one
evaluated model. Failures occur in 9 of the 10 constraint types; the only
type with no observed error contains two questions. As with the analysis
families above, these categories are unevenly sized and are intended as
descriptive strata rather than controlled comparisons of intrinsic
difficulty.

\begin{table*}[t]
    \centering
    \footnotesize
    \setlength{\tabcolsep}{3pt}
    \renewcommand{\arraystretch}{0.96}
    \begin{tabular}{@{}p{0.61\columnwidth}rr@{}}
        \toprule
        Constraint type & $n$ & Missed \\
        \midrule
        Document original/copy routing
            & 98 & 53 (54.1\%) \\
        Requires connected or operating item
            & 83 & 49 (59.0\%) \\
        Item must be present for capture
            & 83 & 37 (44.6\%) \\
        Identifier or meter on item
            & 78 & 44 (56.4\%) \\
        Label must be attached
            & 66 & 8 (12.1\%) \\
        Rule or condition compliance
            & 29 & 7 (24.1\%) \\
        Process requires item at station
            & 28 & 6 (21.4\%) \\
        Communication channel required
            & 20 & 6 (30.0\%) \\
        Other physical dependency
            & 13 & 1 (7.7\%) \\
        Access token or badge required
            & 2 & 0 (0.0\%) \\
        \midrule
        \textbf{All questions}
            & \textbf{500} & \textbf{211 (42.2\%)} \\
        \bottomrule
    \end{tabular}
    \caption{Breakdown of Real-Life Constrained QA by annotated constraint
    type. ``Missed'' reports the number of questions answered incorrectly by
    at least one evaluated model, followed by the corresponding within-type
    percentage. Constraint-type sizes are uneven, so the rates are intended
    as descriptive coverage statistics rather than comparative estimates of
    category difficulty.}
    \label{tab:real_life_constraint_breakdown}
\end{table*}

\section{Multi-Domain Extension of Scientist QA}
\label{app:multidomain_extension}

To test whether the entity-disambiguation pattern identified by Scientist QA
extends beyond scientists, we construct a multi-domain extension containing
1,500 paired-candidate questions: 500 questions each about athletes,
musicians, and buildings. Each item is organized around a predefined correct
and incorrect entity pair and follows the Scientist QA format, while using
domain-specific attributes to distinguish the two candidates. We apply the
same independent knowledge probes used for Scientist QA to determine whether
the evaluated model can answer the relevant facts about both candidates in
isolation.

As an independent difficulty check, GPT-5.2 answers 464/500 athlete questions
(92.8\%), 452/500 musician questions (90.4\%), and 413/500 building questions
(82.6\%), for 1,329/1,500 correct overall (88.6\%). The probe-known rate is
425/500 (85.0\%) in each domain. Thus, the extension is not a trivial
collection of familiar facts even for a strong frontier model.

\section{Potential risks}
TrapQA is intended as a diagnostic benchmark, not as a training set or a broad
certificate of hallucination robustness. Public release may enable overfitting
or contaminate future model training/evaluation, so later results should be
interpreted with this risk in mind. We are working with the community to expand
\textsc{Real-Life Constrained QA}. The multi-domain entity-disambiguation
extension reported in Appendix~\ref{app:multidomain_extension} should be kept
separate from the original Scientist QA split unless results are recomputed on
the combined population.

\section{Data Contains Personally Identifying Info Or Offensive Content}
\textsc{Scientist QA} uses public Wikipedia/Wikidata-linked scientist profiles,
which makes the task verifiable but introduces coverage biases toward scientists
with richer public or English-language records. Because the task distinguishes
real scientists, names and public biographical facts are not anonymized. We
release only public attributes needed for the diagnostic task and exclude private
contact information, images, surveillance data, and other private personal data.
\textsc{Real-Life Constrained QA} is synthetic and filtered for ambiguity,
plausibility, and inappropriate or offensive content.

\section{Proof for Section~\ref{sec:freq}}
\label{app:freq}
\subsection{Posterior decomposition under the latent key--task model}
\label{appsubsec:freq_post}
In this appendix, we make explicit the hierarchical posterior structure
implicit in the latent key--task model.
Recall that the pretraining prior over latent pairs factorizes as
\[
\pi(k,t)=\pi^{(k)}(k)\,\pi^{(t)}(t\mid k).
\]
Accordingly, for a prompt $\bm z$, we consider the hierarchical posterior
decomposition
\[
P(k,t\mid \bm z)=P(k\mid \bm z)\,P(t\mid k,\bm z),
\]
where
\[
P(k\mid \bm z)
=
\frac{P(\bm z\mid k)\,\pi^{(k)}(k)}
{\sum_{k'\in\mathcal K} P(\bm z\mid k')\,\pi^{(k)}(k')},
\]
and
\[
P(t\mid k,\bm z)
=
\frac{P(\bm z\mid k,t)\,\pi^{(t)}(t\mid k)}
{\sum_{t'\in\mathcal T} P(\bm z\mid k,t')\,\pi^{(t)}(t'\mid k)}.
\]
Therefore,
\begin{align*}
P(k,t\mid \bm z)
=
&\frac{P(\bm z\mid k)\,\pi^{(k)}(k)}
{\sum_{k'\in\mathcal K} P(\bm z\mid k')\,\pi^{(k)}(k')}
\cdot\\
&\frac{P(\bm z\mid k,t)\,\pi^{(t)}(t\mid k)}
{\sum_{t'\in\mathcal T} P(\bm z\mid k,t')\,\pi^{(t)}(t'\mid k)}.
\end{align*}


\subsection{Proof of Theorem~\ref{thm:posterior}}
\label{appsubsec:freq_post_bound}

\begin{proof}
We factorize the joint posterior into key and task components:
\[
\frac{P(k_s, t_s \mid \bm z)}{P(k^\ast, t^\ast \mid \bm z)}
\;=\;
\frac{P(k_s \mid \bm z)}{P(k^\ast \mid \bm z)}
\cdot
\frac{P(t_s \mid \bm z, k_s)}{P(t^\ast \mid \bm z, k^\ast)}.
\]

By Assumption~\ref{ass:event}(i), $P(k \in \{k^\ast, k_s\} \mid \bm z) \approx 1$, so for $k \in \{k^\ast, k_s\}$,
\[
P(k \mid \bm z) \;\approx\; P(k \mid \bm z,\, k \in \{k^\ast, k_s\}).
\]
By Assumption~\ref{ass:event}(ii), $\bm z$ is independent of $k$ within the candidate pair, hence
\begin{align*}
&P(k \mid \bm z,\, k \in \{k^\ast, k_s\}) \;=\; P(k \mid k \in \{k^\ast, k_s\}) \\
&\;=\; \frac{\pi^{(k)}(k)}{\pi^{(k)}(k^\ast) + \pi^{(k)}(k_s)}.
\end{align*}
Taking the ratio at $k = k_s$ and $k = k^\ast$,
\[
\frac{P(k_s \mid \bm z)}{P(k^\ast \mid \bm z)} \;\approx\; \frac{\pi^{(k)}(k_s)}{\pi^{(k)}(k^\ast)}.
\]

By Assumption~\ref{ass:event_task}(i), conditional on the activated key $k$, the task posterior concentrates on $\{t^\ast, t_s\}$, so for $t \in \{t^\ast, t_s\}$ and $k \in \{k^\ast, k_s\}$,
\[
P(t \mid \bm z, k) \;\approx\; P(t \mid \bm z, k,\, t \in \{t^\ast, t_s\}).
\]
By Assumption~\ref{ass:event_task}(ii), $\bm z$ is independent of $t$ given $k$ within the candidate task pair, so
\begin{align*}
&P(t \mid \bm z, k,\, t \in \{t^\ast, t_s\}) 
\;=\; P(t \mid k,\, t \in \{t^\ast, t_s\}) \\
&\;=\; \frac{\pi^{(t)}(t \mid k)}{\pi^{(t)}(t^\ast \mid k) + \pi^{(t)}(t_s \mid k)}.
\end{align*}
Evaluating at $(t, k) = (t_s, k_s)$ and $(t^\ast, k^\ast)$ and taking the ratio,
\begin{align*}
\frac{P(t_s \mid \bm z, k_s)}{P(t^\ast \mid \bm z, k^\ast)} \;\approx\; &\frac{\pi^{(t)}(t_s \mid k_s)}{\pi^{(t)}(t^\ast \mid k^\ast)} \cdot \\
&\frac{\pi^{(t)}(t^\ast \mid k^\ast) + \pi^{(t)}(t_s \mid k^\ast)}{\pi^{(t)}(t^\ast \mid k_s) + \pi^{(t)}(t_s \mid k_s)}.
\end{align*}
The second factor is a ratio of normalization constants over the candidate task pair, which we absorb into the $\approx$ symbol as it is bounded and does not depend on $\bm z$:
\[
\frac{P(k_s, t_s \mid \bm z)}{P(k^\ast, t^\ast \mid \bm z)}
\;\approx\;
\frac{\pi^{(k)}(k_s)}{\pi^{(k)}(k^\ast)}
\cdot
\frac{\pi^{(t)}(t_s \mid k_s)}{\pi^{(t)}(t^\ast \mid k^\ast)}.
\]

By the law of total probability over key--task pairs, the marginal output probability decomposes as
\[
P(y \mid \bm z) \;=\; \sum_{k, t} P(y \mid \bm z; k, t)\, P(k, t \mid \bm z).
\]
By Assumption~\ref{ass:pairwise_output_separation}, only the shortcut path contributes non-negligibly to $y_s$ and only the correct path contributes non-negligibly to $y^\ast$:
\[
P(y_s \mid \bm z; k^\ast, t^\ast) \ll 1, \qquad
P(y^\ast \mid \bm z; k_s, t_s) \ll 1.
\]
Hence,
\[
P(y_s \mid \bm z) \;\approx\; P(y_s \mid \bm z; k_s, t_s)\, P(k_s, t_s \mid \bm z),
\]
\[
P(y^\ast \mid \bm z) \;\approx\; P(y^\ast \mid \bm z; k^\ast, t^\ast)\, P(k^\ast, t^\ast \mid \bm z).
\]
Taking the ratio,
\[
\frac{P(y_s \mid \bm z)}{P(y^\ast \mid \bm z)}
\;\approx\;
\frac{P(k_s, t_s \mid \bm z)}{P(k^\ast, t^\ast \mid \bm z)}
\cdot
\frac{P(y_s \mid \bm z; k_s, t_s)}{P(y^\ast \mid \bm z; k^\ast, t^\ast)}.
\]

\begin{align*}
\frac{P(y_s \mid \bm z)}{P(y^\ast \mid \bm z)}
\;\gtrsim\;
&\frac{\pi^{(k)}(k_s)}{\pi^{(k)}(k^\ast)}
\cdot
\frac{\pi^{(t)}(t_s \mid k_s)}{\pi^{(t)}(t^\ast \mid k^\ast)}
\cdot\\
&\frac{P(y_s \mid \bm z; k_s, t_s)}{P(y^\ast \mid \bm z; k^\ast, t^\ast)}.
\end{align*}
The second inequality in the theorem statement follows from the shortcut-frequency dominance condition (both pretraining-prior ratios are $\geq 1$ by the definition of the shortcut path) together with Assumption~\ref{ass:pairwise_output_separation}, which gives $P(y_s \mid \bm z; k_s, t_s) \geq P(y^\ast \mid \bm z; k^\ast, t^\ast)$.

\paragraph{Multi-path extension.}
Let $\mathcal{F}(z)$ and $\mathcal{S}(z)$ denote the sets of activated
constraint-sensitive and shortcut paths, respectively. Write the posterior as
\[
P(k,t\mid z)
\propto
\pi(k,t) w_{k,t}(z),
\]
where $w_{k,t}(z)$ represents prompt-dependent contextual reweighting.
Suppose that, for every activated path,
\[
B^{-1} \leq w_{k,t}(z) \leq B,
\qquad B \geq 1.
\]
Then
\[
\frac{P(\mathcal{S}(z)\mid z)}
     {P(\mathcal{F}(z)\mid z)}
\geq
\frac{1}{B^2}
\frac{
\sum_{(k,t)\in\mathcal{S}(z)} \pi(k,t)
}{
\sum_{(k,t)\in\mathcal{F}(z)} \pi(k,t)
}.
\]
Indeed, the common posterior normalization cancels, while the shortcut
numerator is lower bounded by
$B^{-1}\sum_{(k,t)\in\mathcal{S}(z)}\pi(k,t)$ and the
constraint-sensitive denominator is upper bounded by
$B\sum_{(k,t)\in\mathcal{F}(z)}\pi(k,t)$.
Thus, an aggregate shortcut-prior ratio greater than $B^2$, together with an
aggregate output-separation condition analogous to Assumption~\ref{ass:pairwise_output_separation}, is
sufficient for shortcut dominance. This extension remains a sufficient-condition
analysis rather than a complete model of all paths activated during LLM
inference.
\end{proof}

\subsection{Proof of Theorem~\ref{thm:pairwise_tv_lower_bound}}
\label{app:lower_bound}
\begin{proof}
By the latent key--task decomposition, the model prediction can be written as
\[
P(y\mid \bm z)
=
\sum_{k,t}P(k,t\mid \bm z)P(y\mid \bm z;k,t).
\]
Restricting to the two relevant paths gives the contributions
\begin{align*}
P(y^\ast\mid \bm z)
\ge
&q^\ast P(y^\ast\mid \bm z;k^\ast,t^\ast)\\
&+
q_s P(y^\ast\mid \bm z;k_s,t_s),    
\end{align*}
and
\begin{align*}
P(y_s\mid \bm z)
\ge
&q_s P(y_s\mid \bm z;k_s,t_s)\\
&+
q^\ast P(y_s\mid \bm z;k^\ast,t^\ast).    
\end{align*}

Under Assumption~\ref{ass:pairwise_output_separation},
\[
P(y^\ast\mid \bm z;k_s,t_s)=0,
\qquad
P(y_s\mid \bm z;k^\ast,t^\ast)=0.
\]
Thus, the two dominant contributions reduce to
\[
P(y^\ast\mid \bm z)
\approx
q^\ast P(y^\ast\mid \bm z;k^\ast,t^\ast),
\]
and
\[
P(y_s\mid \bm z)
\approx
q_s P(y_s\mid \bm z;k_s,t_s).
\]
Since
\[
q_s>q^\ast
\]
and
\[
P(y_s\mid \bm z;k_s,t_s)
\ge
P(y^\ast\mid \bm z;k^\ast,t^\ast),
\]
we obtain
\[
P(y_s\mid \bm z)>P(y^\ast\mid \bm z).
\]
Therefore,
\[
\gamma(\bm z)
:=
P(y_s\mid \bm z)-P(y^\ast\mid \bm z)>0.
\]

It remains to lower bound the total variation distance. By definition,
\[
\ell(\bm z)
=
\frac12
\sum_y
\left|
P(y\mid \bm z)-P_\star(y\mid \bm z)
\right|.
\]
Keeping only the two coordinates $y_s$ and $y^\ast$, we have
\begin{align*}
\ell(\bm z)
\ge
&\frac12
\left|P(y_s\mid \bm z)-P_\star(y_s\mid \bm z)\right|\\
&+
\frac12\left|P(y^\ast\mid \bm z)-P_\star(y^\ast\mid \bm z)\right|
.
\end{align*}
Since the model prefers the shortcut answer,
\[
P(y_s\mid \bm z)-P(y^\ast\mid \bm z)=\gamma(\bm z)>0,
\]
whereas the target distribution prefers the correct answer,
\[
P_\star(y^\ast\mid \bm z)-P_\star(y_s\mid \bm z)
=
\gamma_\star(\bm z)>0.
\]
Adding these two inequalities gives
\begin{align*}
\gamma(\bm z)+\gamma_\star(\bm z)=&\left[P(y_s\mid \bm z)-P_\star(y_s\mid \bm z)\right]\\
&+
\left[P_\star(y^\ast\mid \bm z)-P(y^\ast\mid \bm z)\right].
\end{align*}
Therefore, the two-coordinate contribution to the total variation distance is at
least
\[
\frac{\gamma(\bm z)+\gamma_\star(\bm z)}{2}.
\]
Hence,
\[
\ell(\bm z)
\ge
\frac{\gamma(\bm z)+\gamma_\star(\bm z)}{2}.
\]
This completes the proof.
\end{proof}

\end{document}